\documentclass{article}

\usepackage[nonatbib, final]{neurips_2020}

\usepackage[utf8]{inputenc} % allow utf-8 input
\usepackage[T1]{fontenc}    % use 8-bit T1 fonts
\usepackage{hyperref}       % hyperlinks
\usepackage{url}            % simple URL typesetting
\usepackage{booktabs}       % professional-quality tables
\usepackage{amsfonts}       % blackboard math symbols
\usepackage{nicefrac}       % compact symbols for 1/2, etc.
\usepackage{microtype}      % microtypography

\usepackage{amsmath}
\usepackage[export]{adjustbox}
\usepackage{hyperref}
\usepackage{array}
\usepackage{amsthm}
\usepackage{amssymb}
\usepackage{multirow}
\usepackage{bbm}
\usepackage{graphicx}
\usepackage{caption}
\usepackage{bm}
\newtheorem{assumption}{Assumption}
\usepackage{algorithm}
\usepackage[noend]{algpseudocode}
\usepackage{xargs}
\usepackage[pdftex, dvipsnames]{xcolor}
\usepackage[colorinlistoftodos,prependcaption,textsize=tiny]{todonotes}
\newcommandx{\info}[2][1=]{\todo[linecolor=OliveGreen,backgroundcolor=OliveGreen!25,bordercolor=OliveGreen,#1]{#2}}
\usepackage{pifont} 
\usepackage{subcaption}

\newtheorem{theorem}{Theorem}
\newtheorem{lemma}{Lemma}
\newcommand{\xmark}{\ding{55}}
\DeclareMathOperator{\argmax}{argmax}
\title{SMYRF: \\ Efficient Attention using Asymmetric Clustering}
\newcommand{\citeglue}{\cite{GLUE, sst, stsb, qqp, wnli, rte1, rte2, rte3, rte4, mrpc, cola, squad, N18-1101} }
\makeatletter
\newcommand{\settitle}{\@maketitle}
\makeatother
\author{%
  Giannis Daras\\
  Computer Science Department \\
    The University of Texas at Austin\\
  \texttt{giannisdaras@utexas.edu} \\
  \And
  Augustus Odena \\
  Google Research \\
  \texttt{augustusodena@google.com}
  \And
  Nikita Kitaev \\
  Google Research \\
  \texttt{kitaev@cs.berkeley.edu} 
  \And 
  Alexandros G. Dimakis \\
  ECE Department \\
  The University of Texas at Austin \\
  \texttt{dimakis@austin.utexas.edu} \\
}

\begin{document}
\maketitle

\begin{abstract}
% Attention is an indispensable component of many modern models for Natural Language Processing and Computer Vision. The superior performance of attention comes at the cost of quadratic complexity. 
We propose a novel type of balanced clustering algorithm to approximate attention. Attention complexity is reduced from $O(N^2)$ to $O(N \log N)$, where $N$ is the sequence length. Our algorithm, SMYRF, uses Locality Sensitive Hashing (LSH) in a novel way by defining new Asymmetric transformations and an adaptive scheme that produces balanced clusters. The biggest advantage of SMYRF is that it can be used as a drop-in replacement for dense attention layers \textit{without any retraining}.
On the contrary, prior fast attention methods impose constraints (e.g. queries and keys share the same vector representations) and require re-training from scratch. 
We apply our method to pre-trained state-of-the-art Natural Language Processing and Computer Vision models and we report significant memory and speed benefits. Notably, SMYRF-BERT outperforms (slightly) BERT on GLUE, while using $50\%$ less memory. We also show that SMYRF can be used interchangeably with dense attention before and after training. Finally, we use SMYRF to train GANs with attention in high resolutions. 
Using a single TPU, we were able to scale attention to 128x128=16k and 256x256=65k tokens on BigGAN on CelebA-HQ. 
%
%we train BigGAN on Celeba-HQ, with attention at resolution 128x128 and 256x256, capable of generating realistic human faces.

\end{abstract}

\section{Introduction}

Attention layers enable long-range representation learning 
and are becoming indispensable in architectures for both Image Synthesis~\cite{biggan, sagan, daras2019local} and Natural Language Processing~\cite{albert, yang2019xlnet, devlin2018bert, dai2019transformer, t5, liu2019roberta}.
Attention finds further uses in 
other domains like symbolic mathematics and music modeling as 
well ~\cite{MATHTRANSFORMER, musictransformer, child2019generating}.
Unfortunately, attention layers have high computational 
and memory cost which scales quadratically in the size of the input sequence.
This constraint is so onerous that the canonical implementation of
attention for image synthesis - Self-Attention GAN~\cite{sagan} - 
could only afford to use one self-attention layer.
For NLP, modern transformer-based models can only be trained 
in large industry research labs with massive infrastructure investments.
For instance, the recently published GPT-3~\cite{GPT3} model uses $96$ attention layers trained on input sequences of $2048$ tokens.
When fine-tuning pre-trained attention models, 
NLP researchers usually truncate input sentences, 
limiting performance on datasets with longer inputs. 

Recent research~\cite{show_attend_and_tell, daras2019local} 
indicates that dense attention is statistically
and computationally inefficient~\cite{Voita_2019, michel2019sixteen, daras2019local}:
it does not account for the locality inherent in many tasks.
Alternatives have been proposed that are either more
efficient~\cite{child2019generating, adaptively_sparse_transformers, reformer, routing_transformer, sinkhorn, dai2019transformer, lample2019large, star_transformer} 
or that better accommodate locality ~\cite{localattn, daras2019local}.
Most such alternatives have been sparse.
Sparsity can be achieved by limiting attention to pre-defined positions~\cite{localattn, daras2019local, star_transformer, child2019generating}. 
Recent work~\cite{adaptively_sparse_transformers, reformer, routing_transformer, sinkhorn} proposes data-driven sparsity, 
which allows for discovery of arbitrarily complex dependencies 
between input positions. 

Despite this progress, new state-of-the-art 
models~\cite{t5, GPT3, liu2019roberta, clark2020electra, GLUE, superglue}
still use the original dense attention layers.
There are three reasons for this:
(i) alternative fast-attention mechanisms degrade the
performance of the underlying model. 
For example, replacing dense attention layers in Transformers with memory efficient local attention~\cite{localattn} increases perplexity from $41.57$
to $44.23$ \cite{sinkhorn}.
(ii) some mechanisms work well, but make very strict assumptions.
For example, in Star Transformer~\cite{star_transformer} all nodes attend to
a relay node which summarizes the content of the entire input sequence,
but this prevents the use of causal masking, so it can only be used for
encoding.
(iii) some alternatives are only efficient in theory. 
For example, in some variants 
\cite{adaptively_sparse_transformers, Malaviya_2018} sparsification of the
attention map happens after instantiating the matrix, and so quadratic
memory is still used before instantiation. Finally,  
~\cite{child2019generating, beltagy2020longformer} require highly specialized GPU-kernels
and which prevents usage in several hardware settings (e.g. TPUs). The design of fast and efficient attention layers remains a challenge. 

\noindent \textbf{Our Contributions:} \\ 
\textbf{1)} We propose a novel type of balanced clustering to approximate attention. We call the underlying optimization problem Attention Biclustering and prove that finding an exact solution is computationally intractable.\\
\textbf{2)} We propose an algorithm for solving Attention Biclustering efficiently in practice. Our algorithm, SMYRF, uses Locality Sensitive Hashing (LSH) in a novel way by defining new Asymmetric transformations and an adaptive scheme that produces balanced clusters. \\
\textbf{3)} Our method, SMYRF, can handle different query and key vectors, just like normal dense attention.
As a result, SMYRF layers are drop-in replacements for pre-trained models, unlike previously proposed fast-attention mechanisms
such as Sinkhorn~\cite{sinkhorn}, Reformer~\cite{reformer} and Routing Transformer~\cite{routing_transformer}.\\ 
\textbf{4)}
We show through numerous experiments that 
SMYRF attention layers are very effective in terms of performance, memory and speed, even without any training. We measure the memory-performance trade-off of applying SMYRF to state-of-the-art NLP and Computer Vision models, across more than a dozen tasks. For example, we are able to shrink 
the memory requirements 
of a pre-trained BigGAN~\cite{biggan} by $50\%$ while maintaining $98.2\%$ of its Inception score without re-training. \\
\textbf{5)} We finetune SMYRF on GLUE~\cite{GLUE} starting from a BERT (base) checkpoint. We demonstrate that SMYRF-BERT outperforms BERT while using $50\%$ less memory. We also show that with $75\%$ less memory, SMYRF maintains $99\%$ of BERT performance on GLUE. Due to SMYRF's portability, we are also able to conduct experiments for various memory configurations with pre-trained BERT and RoBERTa~\cite{liu2019roberta} models on IMDB. We show slight performance drops for great memory benefits. \\
\textbf{6)} We show that SMYRF can be interchanged with dense layers \textit{before} and \textit{after} training. We report performance gains by using SMYRF in a back-and-forth manner: we replace dense with SMYRF during training (to earn in memory) and we replace SMYRF with dense attention during inference (to earn in performance). The interchangeability of SMYRF with dense attention is unique, as it has not been observed in previously proposed attention alternatives~\cite{reformer, routing_transformer, sinkhorn, beltagy2020longformer, daras2019local}. \\
\textbf{7)} 
We are able to scale the resolution of attention for GANs, due to our reduced memory footprint. We train a BigGAN with an $128 \times 128$ SMYRF attention layer and show it outperforms the dense attention performance, decreasing FID from $26.06$ to $25.03$ in Celeba-HQ-128~\cite{celeba}. Finally, we successfully train a BigGAN with attention at resolution $256\times 256$ on a single v3-8 TPU. \\ 
\textbf{8)} We open-source our code and pre-trained models to encourage more related research: \href{https://github.com/giannisdaras/smyrf}{https://github.com/giannisdaras/smyrf}. 
\section{Background}

Attention~\cite{vaswani2017attention} works by computing inner products of query and key vectors.  Depending on the application, these vectors may represent embeddings for tokens or image pixels. Input of each attention layer is three sets: $\mathcal Q, \mathcal K, 
\mathcal V$ for query, key and value vectors respectively.
Attention of $q$ to the keys set $\mathcal K$ outputs a new vector $o_q$ , which is a weighted sum of value vectors $v_i \in \mathcal V$ where each weight $w_i$ increases with the inner product $q \cdot k_i$. Specifically, the output is computed as:
\begin{equation}
o_q = \sum_{i=1}^{N} w_i v_i, \qquad w_i = \frac{e^{q \cdot k_i}}{\sum_{j=1}^{N} e^{q \cdot k_j}}.
\label{dense}
\end{equation}
Here, we assumed for notational simplicity that $N = |\mathcal Q| = |\mathcal K|$. Using matrix notation, attention is equivalently defined as $\sigma(Q \cdot K^T) \cdot V$ where $Q, K, V$ are matrices with rows the embeddings for each query, key, value and the function  $\sigma(.)$ computes the row-wise softmax.

\section{Approximating Attention with Clustering}
\subsection{Motivation}

Our method is motivated by the observation that attention matrices have interesting structure in real datasets.
Naively, to compute dense attention, as  equation \ref{dense} shows, we need to compute all outputs $o_{q_i}$, i.e. $O(|\mathcal Q| \cdot |\mathcal K|)$, a quadratic number of inner products $q_i \cdot k_j, \ q_i \in \mathcal Q, \ k_j \in \mathcal K$. However, we observe that in most real networks, the attention weights $w_i$ are sparse, because of the softmax operation and the structure of the vectors. 
For example we observe that in a pre-trained BigGAN on ImageNet, on average $\bm {98.11 \pm 0.26 \%}$\footnote{The reported numbers are calculated by inspecting the attention maps of 1000 random generated images.} of keys get weight less than $0.01$ in softmax and $\bm {86.11 \pm 2.92 \%}$ of them get less than $\frac{1}{|\mathcal K|}$, where $\mathcal K$ is the number of keys.

Further, we observe that the attention matrix is near low-rank, even after the softmax. 
By definition, the matrix $Q \cdot K^T$ is going to be of rank at most the dimension of the query and key vectors. Therefore, if the embeddings dimension is smaller than the input sequence, the attention matrix is low-rank. This is more pronounced for images and long-context language models. 
However, one can easily construct cases of low-rank matrices which become full rank after softmax. Our finding is that this does not happen in practice. In the Appendix we show that \textit{real attention matrices of pretrained models have a sharp decay in their singular values and hence can be well approximated by low-rank matrices}. 

SMYRF benefits from sparsity and low-rank structure of attention matrices.
By clustering keys and queries into groups, we obtain block-diagonal structure in the approximate attention matrix, since only query-key pairs within the same cluster are computed. We show that this method leads to accurate approximations of dense attention and it can be computed much faster and with much less memory.

\subsection{Problem Formulation}
We formulate the assignment of keys and queries into clusters as an optimization problem. 
Denote with $P_{ij} = q_i^T k_j$ the element $(i, j)$ of the product matrix $P = Q \cdot K^T$ and the attention map with $M = \sigma(Q \cdot K^T)$. 
We will assign query and key vectors into $L$ clusters $c_1, c_2, ..., c_L$ and compute attention only within each cluster. 
For fast execution on TPUs/GPUs, all partial attentions should be computed in parallel. 
For this reason, we require that clusters are balanced: i.e. all clusters contain the same number of keys and queries. We note that the number of keys in each cluster does not have to be equal to the number of queries. 
Formally, each cluster contains  $\frac{|\mathcal Q|}{L}$ queries and $\frac{|\mathcal K|}{L}$ keys. 

We denote with $\mathcal C^L$ the set of all possible assignments in $L$ balanced non-overlapping clusters. A specific assignment is denoted by $\mathcal C^L_t$ 
and there are $T$ possible such assignments, where $T$ is exponentially large in the number of keys and queries. 
$$
\mathcal C^L = \{\mathcal C_1^L,  C_2^L, ... \mathcal C_T^L\}.
$$
\begin{equation}
\mathcal C_t^L=\{c_1, c_2, ..., c_L\}: \quad   
\begin{cases}
 c_i = \{q_1, ..., q_{\frac{|\mathcal Q|}{L}}, k_1, ...,  k_\frac{|\mathcal K|}{L}\}, \quad c_i \subseteq \mathcal Q \cup \mathcal K, \ \ \forall i\in \{1, ..., L\} \\ 
c_x \cap c_y = \varnothing \quad \forall c_x, c_y \in \mathcal C_t^L.
\end{cases}
\label{balanced}
\end{equation}

We emphasize that every key and query is assigned in a unique cluster for any valid assignment 
$\mathcal C^L_t$:  
$c_x \cap c_y = \varnothing \quad \forall c_x, c_y \in \mathcal C_t^L.$ We also define a masking operator $\textrm{Mask}_\epsilon$ that takes as input: (i) a clustering $\mathcal C_t^L \in \mathcal C^L$ and (ii) the product matrix $P$ and replaces $(q, k)$ pairs that are not in the same cluster with $-a$, where $a \in \mathbb R^{+}$ is a constant chosen to satisfy $e^{-a} = \epsilon$ for a given $\epsilon \geq 0$. Formally:
$$
\textrm{Mask}_\epsilon(\mathcal C_t^L, P_{ij}) = \begin{cases}
P_{ij} \quad \textrm{iff } \exists t: (i, j) \in c_t, \\
-a, \quad \textrm{o/w}.
\end{cases}
$$

Intuitively, the masking operator replaces inner products of queries and keys that are not in the same cluster with an arbitrarily small number, so that the softmax will assign a score arbitrarily close to zero to these entries.
We denote with $\hat P_\epsilon=\textrm{Mask}_\epsilon(\mathcal C_t^L, P)$ the product matrix after the masking. 
With this notation, $\hat P_0 = \textrm{Mask}_{0}( \mathcal C_t^L, P)$, is the product matrix for the within-clusters attention. 

\noindent \textbf{Attention Biclustering:}
Under this formulation, we are searching for the cluster assignment $\mathcal{C}^L_t$ that approximates the dense attention matrix $\sigma(P)$ as well as possible, in Frobenius norm:
\begin{equation}
\min_{\mathcal C_t^L \in \mathcal C^L} || \sigma(\hat P_0)- \sigma(P)||_F.
\label{min_problem}
\end{equation}
Note that $L$ must divide the number of queries and keys for this problem to be well-defined.

% \begin{figure}[!htp]
%     \centering
%     \includegraphics[width=0.90\textwidth]{figs/queries_keys.png}
%     \caption{Left: dense attention where all queries attend to all keys. Right: SMYRF approximation. Queries and keys are clustered in groups of equal size and attention is performed within each group.}
%     \label{fig:subattn}
% \end{figure}

\subsection{Complexity of Attention Biclustering}
We start by showing that Attention Biclustering, the optimization problem defined in (\ref{min_problem}), is provably computationally intractable.

\begin{theorem}
Attention Biclustering (\ref{min_problem}) is NP-hard.
\label{main_theorem}
\end{theorem}
We defer the proof of this theorem to the Appendix. 
Our proof proceeds by first establishing hardness before the softmax, using a reduction from three dimensional matching~\cite{gary}. 
We then leverage this to establish hardness of approximating attention through clustering after the softmax operation. 

We consider it interesting to establish the computational intractability of
Attention Biclustering, since this clustering formulation is quite unique due to the softmax operation. Our hardness result rules out an exact polynomial solution, unless P=NP. 
We propose an efficient algorithm that leverages hashing to assign queries and keys to clusters. 
Formally proving an approximation guarantee or provable inapproximability for the attention approximation problem we proposed remains open. 

\subsection{Proposed algorithm: SMYRF}

Our algorithm consists of the following steps:\\
\textbf{1)} We first propose novel asymmetric transformations $F, G: \mathbb R^{d} \to \mathbb R^{d'}$ such that for all given queries $q_1, q_2 \in \mathcal Q$ and keys $k \in \mathcal K$: $q_1 \cdot k \leq q_2 \cdot k \iff ||F(q_1) - G(k)||_2 \leq ||F(q_2) - G(k)||_2$. \\
\textbf{2)}
We then use a Locality Sensitive Hashing (LSH) function $h:\mathbb R^{d'}\to \mathbb R$ to map transformed vectors in real numbers, so that that vectors that are close in Euclidean distance correspond to numbers that are close on the real line. \\
\textbf{3)} 
We sort vectors based on their LSH value and group them
by adapting the thresholds to ensure $L$ balanced clusters. \\
\textbf{4)}
We perform dense attention within each cluster.

Our approximate attention algorithm relies on a few technical innovations:

\noindent \textbf{Novel Asymmetric Transformations:}
We need an efficient way to find, for any given query vector $q_i \in \mathcal Q$ the set of keys with which it has big inner products. 
This problem, called Maximum Inner Product Search (MIPS), can be efficiently solved  by transforming query and key vectors to convert it to a Nearest Neighbor Search (NNS) as proposed in the pioneering Asymmetric LSH (Locality Sensitive Hashing) work by Shrivastava et al.~\cite{l2lsh}.

%Instead of finding big inner products, \cite{l2lsh} proposes to use separate functions for queries and keys, chosen to transform the problem of Maximum Inner Product Search (MIPS) to Nearest Neighbors Search (NNS).

We are looking for functions $F: \mathbb R^{d}\to \mathbb R^{d'}, G: \mathbb R^d \to \mathbb R^{d'}$ such as: $||F(q) - G(k)||_2^2 = D(q\cdot k), \ \forall (q, k)$ where $D:\mathbb R \to \mathbb R$ a decreasing function that depends only on the inner product $q\cdot k$. We constrain our focus on functions $D$ that decrease linearly with the inner product $q\cdot k$. Several previous works have proposed Asymmetric LSH transformations~\cite{l2lsh,xbox,h2lsh} but focus on the case where we have a \textit{single query} $q$ and multiple keys. In that case, any norm $||q||_a$ where $a=\{1, ..., \infty\}$ is constant and thus $D = D(q\cdot k, ||q||_a)$.  

Our central algorithmic contribution is the proposal of novel asymmetric functions:
\begin{equation}
    F(q_i) = \left[q_i; 0; \sqrt{M_Q^2 + M_K^2 - ||q_i||_2^2} \right], \qquad G(k_i) = \left[k_i; \sqrt{M_Q^2 + M_K^2 - ||k_i||_2^2}; 0\right]
\end{equation}
where we use the constants $M_Q = \max_{q_i}||q_i||_2, \quad M_K = \max_{k_i}||k_i||_2$, or any other upper bound on the norms. 
With this transformation, all queries and keys are mapped to a $(d+2)$-dimensional ball with radius $\sqrt{M_Q^2 + M_K^2}$ and the distance of the transformed vectors decreases linearly with the inner product of the original vectors:
\begin{equation}
    ||F(q_i) - G(k_i)||_2^2 = 2 \cdot \left( M_Q^2 + M_K^2 -  q_i\cdot k_i\right).
\end{equation}
Note that the Euclidean distance of the transformed vectors depends only on the inner product of the original vectors and not on individual norms $||q_i||_2$ as in previous work~\cite{e2lsh, h2lsh, xbox}. We include details of comparison to the numerous prior asymmetric transformations in the Appendix.

\noindent \textbf{Adaptive Clustering:}
The final step of SMYRF is to use the hashed values to create \textit{balanced} clusters. These are created by forming balanced hash buckets where every group is assigned the same number of query and key vectors. We modify the E2LSH~\cite{e2lsh} hashes to create balanced clusters as follows: 
Instead of rounding the E2LSH to an integer value as in~\cite{e2lsh}, we adaptively set the boundaries of the $1$-d hashed space to ensure the same number of query and key vectors per interval. Computationally wise, this only requires sorting the hashes. We explain the mathematical details of our adaptive clustering scheme and the differences with E2LSH in the Appendix.

\noindent \textbf{Computational Complexity and speedups:} For notational simplicity we assume $|\mathcal Q| = |\mathcal K| = N$.
The total time and memory complexity of SMYRF is $O\left(H \cdot N \cdot \log N + H \cdot \frac{N^2}{L}\right)$, where: $H$ denotes hashing rounds, $N$ number of query/key vectors and $L$ number of clusters. For most of our experiments we choose $L = O(N), \ H=O(1)$, and thus complexity is $O(N\log N)$. Even though we obtain optimal complexity for $L=O(N), \ H=O(1)$, both $L, H$ are parameters that can be tuned to satisfy the desired memory-performance trade-off. Regarding speed, SMYRF accelerates 
a lot attention as sequence length increases. For example, for sequence length 2048, SMYRF-BERT offers $\approx 20\%$ speedup, while for $4096$ speedup increases to $\approx 50\%$. We include detailed speed plots for applying SMYRF to BERT in the Appendix.

\section{Experiments}

\subsection{Pre-trained models}
We first illustrate that SMYRF is an excellent drop-in replacement for pre-trained dense attention. We show significant memory benefits for relatively small performance drop, \textit{with no training at all}. 
We use a pre-trained\footnote{Since BigGAN's official checkpoints are not publicly available, we use the authors' open-source, PyTorch~\cite{paszke2019pytorch} pre-trained models: https://github.com/ajbrock/BigGAN-PyTorch
} BigGAN, which is a state-of-the-art model in Image Generation for ImageNet~\cite{ImageNet}. BigGAN has a single attention layer at resolution $64\times64$ (4096 queries). We replace BigGAN's dense attention with a SMYRF layer at the same resolution, with no other modifications. Figure \ref{image_quality} illustrates images generated by SMYRF-BigGAN for different memory savings, ranging from $99.44\%$ (first column) to $50\%$ (one to last column). Last column shows generated images using the dense attention layer ($100\%$ memory). As shown, SMYRF 
enables a new tradeoff in the design space. 
We can drastically reduce attention memory by $93.75\%$ with a small degradation or select any other point in this tradeoff depending on hardware specifications. 
We report a few Inception~\cite{inception_score} and FID~\cite{FID} scores for different memory savings in Table \ref{biggan_pretrained}. We emphasize that no further modification was made to this model other than replacing the attention layer. By shrinking $50\%$ the memory requirements of attention, SMYRF maintains $98.2\%$ of Inception performance without any training. In the Appendix, we also include visualizations of clustering assignments in real-world images.

\begin{figure}[!htp]
\begin{minipage}{0.69\textwidth}
    \includegraphics[width=8cm]{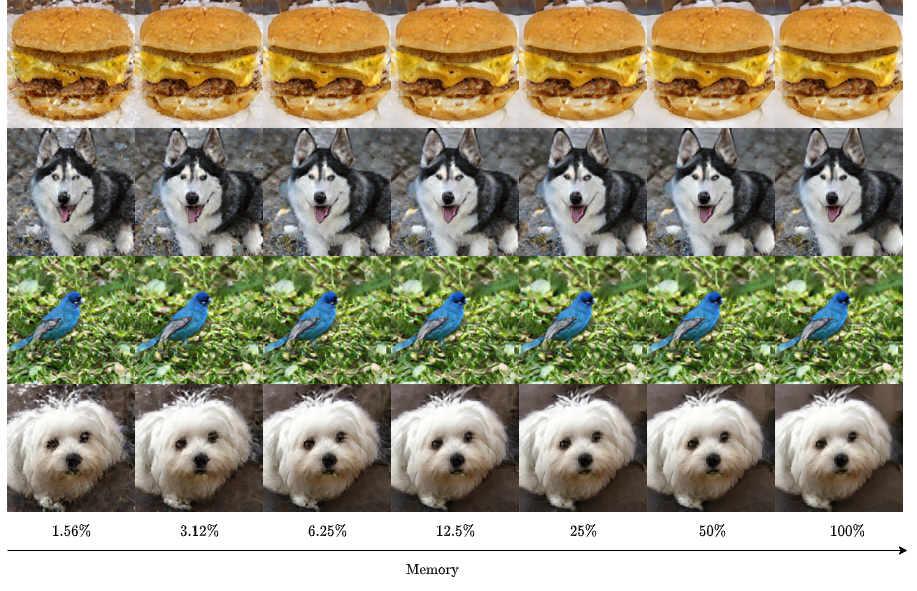}
    \captionof{figure}{\small Images generated by SMYRF-BigGAN. \\
    The model is initialized with the weights of a pre-trained BigGAN \\ on ImageNet (no further training). \\ We show images for  memory reduction ranging from $98.44\%$ \\ (first column) to $50\%$ (one to last column). \\ Last column shows generated images by BigGAN with dense attention.}
    \label{image_quality}
\end{minipage}
\begin{minipage}{0.30\textwidth}
    \includegraphics[width=5cm]{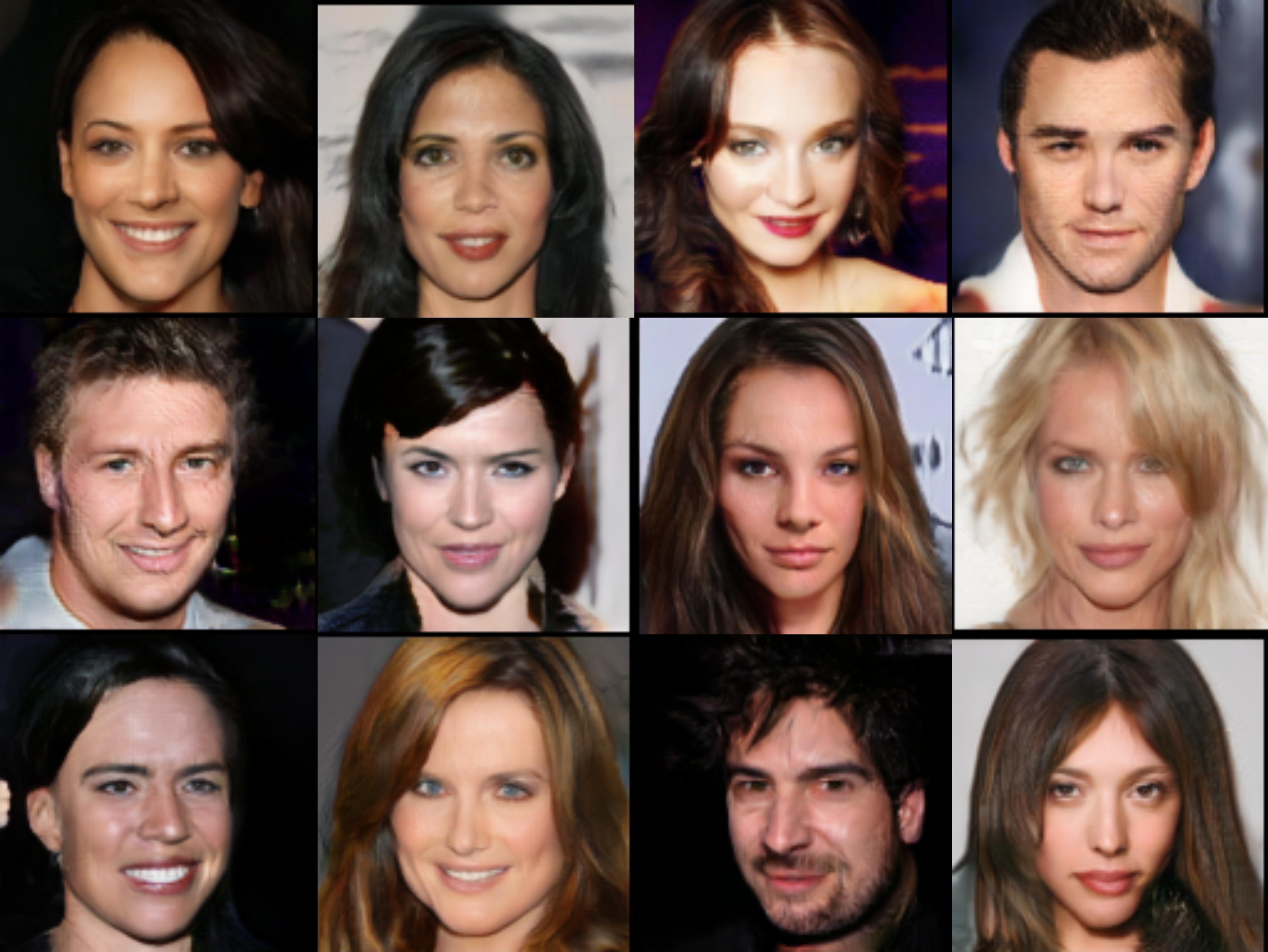}
    \caption{\small Generated images from SMYRF-BigGAN on Celeba-HQ-128.  Attention at $128\times128$. The trained model uses $50\%$ less memory compared to original BigGAN.}
    \label{gen128}
\end{minipage}
\end{figure}

\begin{table}[!htp]
\begin{center}
\begin{tabular}{l|l|l|l|l|l}
     & Memory & Rounds & $C$ & Inception & FID \\ \hline
    
    BigGAN & \multirow{1}{4em}{$100\%$} & 1 & 4096 & $\bm{93.79 \pm 1.96}$ & 11.30 \\ \hline
    
    % \multirow{5}{4em}{}
    % & 2 & 2048 & $91.23 \pm 1.87$ & $11.46$ \\
    % & & & 4 & 1024 & $90.58 \pm 2.45$ & $11.72$ \\
    % & & & 8 & 512 & $91.28 \pm 2.04$ & $11.65$ \\
    % & & & 16 & 256 & $92.89 \pm 2.23$ & $11.61$ \\ 
    % & & & 32 & 128 & $\bm{92.95 \pm 2.24}$ & $11.63$ \\
    % & & & 64 & 64 & $92.34 \pm 1.65$ & $11.54$ \\
    % & & & 128 & 32 & $91.84 \pm 2.00$ & $\bm{11.52}$ \\
    \multirow{9}{4em}{SMYRF-BigGAN} & \multirow{3}{4em}{50\%} & 
    % \multirow{6}{4em}{98.2\%} & 
    % & 1 & 2048 & $85.07 \pm 2.22$ & $12.48$ \\
    % & & & 2 & 1024 & $86.54 \pm 1.38$ & $12.52$ \\ 
    % 4 & 512 & $88.13 \pm 2.74 $ & $12.61$ \\
    % & & 8 & 256 & $89.42 \pm 1.57$ & $12.33$ \\ 
    % & & 16 & 128 & $90.86 \pm 1.76$ & $12.27$ \\ 
    32 & 64 & $91.91 \pm 2.65$ & $12.18$ \\
    & & 64 & 32 & $\bm{92.09 \pm 1.83}$ & $12.18$ \\
    & & 128 & 16 & $91.59 \pm 1.83$ & $\bm{12.10}$ \\ 
    
    \cline{2-6}
    & 
    \multirow{3}{4em}{$25\%$} &
    % \multirow{6}{4em}{$95.5\%$} &
    % 1 & 1024 & $73.61 \pm 1.62$ & $15.1$  \\
    % & & & 2 & 512 & $79.54 \pm 2.21$ & $14.38$ \\
    % 4 & 256 & $84.08 \pm 2.09$ & $14.11$ \\
    % & & 8 & 128 & $85.98 \pm 1.93$& $13.87$ \\
    % & & 16 & 64 & $87.50 \pm 1.95$& $13.50$ \\
    32 & 32 & $87.90 \pm 1.90$ & $13.34$ \\
    & & 64 & 16 & $88.45 \pm 1.70 $& $13.45$ \\
    & & 128 & 8 & $\bm{89.61 \pm 1.63}$ & $\bm{13.19}$ \\
    \cline{2-6}
    &  
    \multirow{3}{4em}{$12.5\%$} & 
    % \multirow{6}{4em}{$88.4\%$} &
    % 1 & 512 & $63.16 \pm 1.30$ & $18.80$ \\
    % & & & $2$ & $256$ & $71.46 \pm 1.96$ & $17.73$ \\
    % $4$ & 128 & $75.67 \pm 1.21$ & $17.21$ \\ 
    % & & 8 & 64 & $79.00 \pm 1.80$ & $16.64$ \\
    % & & 16 & 32 & $81.65 \pm 1.55$ & $16.26$ \\
    32 & 16 & $81.67 \pm 1.97$  & 16.08 \\
    & & 64 & 8 & $\bm{82.87 \pm 1.82}$ & $\bm{16.00}$ \\
    & & 128 & 4 & $82.10 \pm 2.06$ & $16.03$ \\ 
     \bottomrule
\end{tabular}
\end{center}
\caption{Effect of SMYRF attention approximation on a pre-trained BigGAN (with no training). Rounds denote the number of LSH hashes and $C$ the number of queries per cluster.}
\label{biggan_pretrained}
\end{table}

\subsection{Finetuning pre-trained models}

In this section, we \textit{finetune} pre-trained models with SMYRF. We show that finetuned SMYRF models, with $50\%$ memory reduction, can outperform dense attention. We also show that even with more aggressive memory-shrinking, up to $97\%$, SMYRF maintains a relatively good performance.

We train SMYRF-BERT (base) on GLUE \citeglue 
benchmark, using sequence length 128. 
We compare the following five models: (i) BERT~\cite{devlin2018bert} (base), (ii) SMYRF-BERT (base) with $50\%$ memory reduction (2nd row), (iii) SMYRF-BERT (base) with $25\%$ memory reduction (3rd row), (iv) BERT (base) with input sequences truncated to $64$ tokens ($50\%$ memory reduction, 4th row), (v) BERT (base) with input sequences truncated to $32$ tokens ($25\%$ memory reduction, 5th row). We summarize results on Table \ref{glue_results}. Remarkably, SMYRF-BERT (slightly) \textbf{outperforms} original dense attention, while using $50\%$ less memory. We also underline that SMYRF-BERT with $25\%$ of original memory, maintains $\approx \bm{99\%}$ of original model performance, while the BERT-model that uses the same memory (last row) maintains only $\approx 89\%$. 

To demonstrate that SMYRF scales for larger models, we also run experimetns with \textbf{SMYRF-BERT large} to a subset of the GLUE tasks. Specifically, SMYRF-BERT large obtains $\mathbf{60.4\%}$ performance (Matthew's Correlation) in the CoLA task and $\mathbf{90.2\%}$ (accuracy) in the QQP task. Both scores are significantly improved compared to the scores of the SMYRF-BERT base model, which shows that the approach scales to models with more attention layers. The corresponding scores of BERT large are $60.5\%$ and $89.3\%$ which are on par with the SMYRF performance for that model.

Since GLUE~\cite{GLUE} datasets contain mostly short inputs, we also experiment on the IMDB~\cite{IMDB} dataset, using sequence length $512$ tokens\footnote{Note that for fair comparison with dense attention, we train SMYRF layers and dense layers on the same sequence length, following the comparison scheme of Reformer~\cite{reformer}. As previous work has shown~\cite{beltagy2020longformer}, training on IMDB  (and other long-input datasets) with bigger sequence length can help performance.}. We experiment with SMYRF-BERT (base) and we report results for various configurations (memory savings, hashing rounds and cluster size). To support our argument that our algorithm is a drop-in replacement to \textit{any} dense attention layer, we also include some results for RoBERTa~\cite{liu2019roberta} (base).
Results are summarized in Table \ref{nlp_finetuning}. Notably, 
SMYRF-BERT maintains $97.2\%$ of dense attention performance for $87.5\%$ memory savings. 

In Table \ref{backandforth}, we provide results 
for a Back-and-Forth procedure: we finetune with 
SMYRF and then for inference we use dense attention. By doing that, we observe performance almost equivalent to training with dense attention, while saving computational resources with SMYRF training. This indicates interchangeability between SMYRF and dense attention, which has not been previously reported. We use it to
to train in a memory efficient manner and obtain maximum final performance.

\begin{table}[!htp]
    \begin{adjustbox}{width=\columnwidth, center}
    \begin{tabular}{l|l|l|l|l|l|l|l|l|l|l|l}
    & Avg. & $\#$ & $C$ & CoLA & MNLI-m/mm & MRPC & QNLI & QQP & RTE & SST-2 & STS-B \\ \hline
    
    BERT$_{128}$ & $82.69$ & 1 & 1 & 57.83 & $84.43/\bm{84.68}$ & $\bm{88.41}$ & $\bm{91.31}$ & 89.70 & 65.70 & $\bm{93.46}$ & $88.73$ \\ \hline
    \multirow{2}{4em}{SMYRF-BERT} & $\bm{83.12}$ & 2 & 32 & $58.79$ & \textbf{85.02}/84.27 & 87.69 & 91.14 & $\bm{89.72}$ & $\bm{68.59}$ & 93.23 & $\bm{89.65}$ \\ \cline{2-12}
    & $81.74$ & 2 & 16 & $\bm{58.90}$ & $82.86/83.49$ & $85.72$ & $89.53$ & $89.33$ & $64.98$ & $93.12$ & $87.75$ \\ \hline
    BERT$_{64}$ & $81.57$ & 1 & 64 & 58.80 & 82.34/82.47 & 87.02 & 90.48 & 89.69 & 61.73 & 93.00 & 88.64 \\ \hline
    BERT$_{32}$ & $73.56$ & 1 & 32 & $56.40$ & $64.51/63.41$ & $77.89$ & 79.81 & 88.59 & 55.23 & 92.66 & 83.53 \\ 
    \bottomrule
    \end{tabular} 
    \end{adjustbox}
    \caption{Results on GLUE~\cite{GLUE} (dev). $\#:$ hashing rounds. $C:$ the number of queries per cluster. SMYRF outperforms BERT while using $50\%$ less memory in each of the $12$ attention layers.}
    \label{glue_results}
\end{table}

\begin{table}[!htp]
\begin{center}
\begin{tabular}{l|l|l|l|l|l}
     & Dataset & Memory & Accuracy & Rounds & Cluster  \\ \hline
     BERT & \multirow{8}{4em}{IMDB} & $100\%$ & $\bm{94.12\%}$ & 1 & 512 \\\cline{1-1} \cline{3-6}
    \multirow{5}{4em}{SMYRF-BERT} & 
    &  
    \multirow{1}{4em}{50\%} & 
    % $90.45\%$ & 1 & 256 \\ 
    % & & & $91.15\%$ & 2 & 128 \\
    % $92.30\%$ & 4 & 64 \\ 
    $92.64\%$ & 8 & 32 \\ \cline{3-6}
    & & \multirow{1}{4em}{25\%} & 
    % $92.50\%$ & 8 & 16 \\ 
    $92.52\%$ & 16 & 8 \\
    \cline{3-6}
    
    & & \multirow{1}{4em}{$12.5\%$} & 
    % $90.99\%$ & 4 & 16 \\
    $91.46$ & 8 & 8 \\
    % & & & $90.23$ & 16 & 4 \\ 
    \cline{3-6}
    
    & & \multirow{1}{4em}{6.25\%} &  
    % $87.74\%$ & $4$ & 8 \\
    $88.78\%$ & $8$ & $4$  \\
    % & & & $82.91\%$ & $16$ & 2 \\ 
    \cline{3-6} 
    & & \multirow{1}{4em}{3.125\%} & $87.49\%$ & 4 & 4 \\
    \cline{1-1} \cline{3-6}
    \multirow{1}{4em}{RoBERTa} & & $100\%$ & $\bm{94.96\%}$ & 1 & 512 \\
    \cline{1-1} \cline{3-6}
    
    SMYRF-RoBERTa & & $50\%$ & $93.72$ & 8 & 32 \\
    % BERT &
    % \multirow{4}{4em}{BoolQ} & 
    % $100\%$ & 
    % $\bm{73.64\%}$ &
    % $1$ & 
    % $128$ \\ \cline{1-1} \cline{3-6}
    % \multirow{3}{4em}{SMYRF-BERT} &
    % &
    % \multirow{3}{4em}{$50\%$} & 
    % $69.24\%$ &
    % $1$ &
    % $64$ \\ 
    % &
    % & 
    % &
    % $69.00\%$ &
    % $2$ &
    % $32$ \\ 
    % &
    % &
    % & 
    % $70.28\%$ & 
    % $4$ &
    % $16$ \\
    % & & & $70.83\%$ & 8 & 8 \\ 
    \bottomrule
\end{tabular}
\end{center}
\caption{Finetuning BERT~\cite{devlin2018bert} (base) and RoBERTa~\cite{liu2019roberta} (base) on IMDB dataset for various configurations. For SMYRF models, we \textbf{train} and \textbf{evaluate} with SMYRF.}
\label{nlp_finetuning}
\end{table}

\begin{table}[!htp]
    \centering
    \begin{tabular}{l|l|l|c|l}
    & Dataset & Memory & SMYRF Inference & Accuracy \\ \hline
    RoBERTa & \multirow{6}{4em}{IMDB} & $100\%$ & \xmark 
    & $\bm{94.96\%}$ \\ \cline{1-1} \cline{3-5}
    \multirow{2}{4em}{SMYRF-RoBERTa} & & \multirow{2}{4em}{$50\%$} & \xmark & $93.72\%$ \\ 
    & & & \checkmark & $\bm{94.62\%}$ \\ 
    \cline{1-1} \cline{3-5} 
    BERT & & $100\%$ & \xmark & $94.12\%$ \\ \cline{1-1} \cline{3-5}
    \multirow{2}{4em}{SMYRF-BERT} & & \multirow{2}{4em}{$50\%$} & \xmark & $92.64\%$ \\ 
    & & & \checkmark &  $\bm{93.54\%}$\\ 
    \bottomrule
    \end{tabular}
    \caption{Interchangeability of SMYRF and dense attention. We \textbf{train} with SMYRF and \textbf{evaluate} with \textit{dense attention} for lightweight training and maximum performance.}
    \label{backandforth}
\end{table}

\subsection{Training from scratch}
We also include experiments for networks trained from scratch.
This shows that a non-pretrained model can learn with randomly initialized, SMYRF layers.  Initially, the random weights produce less sparsity. However, the model quickly learns to create sparse attention maps and learning under our framework is possible. We use BigGAN~\cite{biggan} as the underlying model (see Appendix for details).
We conduct our experiments on Celeba-HQ~\cite{celeba}, which contains $30$K images of celebrities at resolution $1024\times 1024$. We choose Celeba-HQ because: (i) images are in resolution higher than $128\times 128$, (ii) our budget is limited and Celeba-HQ requires much less training steps compared to ImageNet~\cite{ImageNet}. With SMYRF, we move attention from $64\times 64$ resolution to $128\times128$ and train with $50\%$ less memory than dense attention. In Table \ref{SMYRF128res}, we report FID for BigGAN and SMYRF-BigGAN after $120$K steps training on Celeba-HQ-128 (downsampled to $128\times 128$). SMYRF-BigGAN \textit{outperforms} BigGAN's FID by $\bm{3.95\%}$. Generated images from our model are shown in Figure \ref{gen128}. We finally move the attention layer to resolution $256\times 256$ (65k length) and we successfully train on Celeba-HQ-256 for 120K steps on a single TPU v3-8. As far as we know, no other GAN has been trained with attention in higher resolution than this. Details and generated images are included in the Appendix.

\begin{table}[!htp]
\begin{center}
\begin{tabular}{l|l|l|l|l|l|l}
& Resolution & Attention & Memory & Rounds & $C$ & FID \\ \hline
BigGAN & \multirow{2}{4em}{$128\times 128$} & \multirow{1}{4em}{$64\times64$} & $100\%$ & $1$ & $4096$ & 26.06 \\ \cline{1-1} \cline{3-7}
\multirow{1}{8em}{SMYRF-BigGAN} & & $128\times 128$ & $50\%$ & 4 & 2048 & $\bm{25.03}$\\ \bottomrule
\end{tabular}
\caption{Results on BigGAN training on Celeba-HQ-128 for 120K steps. Moving attention from $64\times 64$ to $128 \times 128$ helps performance: FID decreases from $26.06$ to $\bm{25.03}$. Memory percentages in this Table have as reference the memory a dense attention layer would use at the given resolution.}
\label{SMYRF128res}
\end{center}
\end{table}

\subsection{Comparison with other efficient attention techniques}
\begin{table}[!htp]
\centering
\begin{tabular}{c|c}
Model & IMDB (3 epochs)\\ \hline
SMYRF-RoBERTa & \textbf{93.7\%}   \\ 
E2LSH & 89.3\% \\
Reformer & 88.7\% \\
\bottomrule
\end{tabular}
\caption{LSH ablation experiment. The E2LSH model corresponds to the SMYRF-RoBERTa model using the E2LSH~\cite{e2lsh} hashing scheme instead of the asymmetrical transformations. The Reformer model corresponds to running SMYRF-RoBERTa with the cross polytope LSH~\cite{nearoptimallsh} scheme, which is used in the Reformer~\cite{reformer} paper.}
\label{eff_attn_comp}
\end{table}

To validate the effectiveness of the proposed asymmetrical transformations, we replace SMYRF's hashing scheme with the E2LSH~\cite{e2lsh} scheme and the cross-polytope LSH~\cite{andoni2015practical} scheme of the Reformer and we evaluate all models on the IMDB~\cite{IMDB} dataset, after training for three epochs. The results are summarized in Table \ref{eff_attn_comp}. As shown, the asymmetrical transformations of SMYRF largely outperform all the other LSH schemes. This is expected since by design SMYRF tries to form clusters that maximize the inner products between queries and keys, while E2LSH and Reformer try to minimize euclidean distance and angular distance respectively, which is not the best objective when dealing with queries and keys with different vector representations and arbitrary norms. 

To compare with the Longformer~\cite{beltagy2020longformer}, we evaluate SMYRF on the Hyperpartisan News Detection~\cite{kiesel-etal-2019-semeval} dataset. For this task, Longformer reports $94.8\%$ accuracy with $4096$ context-length. SMYRF obtains $\textbf{97.2\%}$ performance while only using $512$ tokens. Longformer slightly outperforms (for $\approx 1\%$) SMYRF in the IMDB dataset but it uses $8$ times more tokens to achieve that. Unfortunately, the available RoBERTa~\cite{liu2019roberta} models have been trained with maximum positional embeddings at $512$ tokens and thus we cannot determine whether bigger sequence lengths would favor SMYRF. Nevertheless, SMYRF performs on par with other efficient attention techniques without requiring any pre-training.

\section{Related work}

The fact that attention maps of pre-trained layers are sparse is well-known~\cite{Voita_2019, michel2019sixteen, daras2019local, adaptively_sparse_transformers, Sukhbaatar_2019, Peters_2019}. Relevant research to our work includes efforts to leverage that sparsity by limiting attention of each element to a subset of the original sequence. ~\cite{localattn} proposes to limit attention to a sliding window around each element. Even though this simple idea is a strong baseline due to locality, this method is usually outperformed~\cite{sinkhorn, reformer, routing_transformer} by data-driven methods for assigning to each query the keys it will attend to. One recent research work that performs well with pre-defined sparsity is Longformer~\cite{beltagy2020longformer}. Longformer has been shown to perform well in downstream tasks after pre-training for $65$K gradient steps, resuming MLM training of a pre-trained RoBERTa~\cite{liu2019roberta} model. However, this work requires custom GPU kernels that do not transfer across hardware (i.e. are not efficient on TPUs). SMYRF differs from Longformer in other important aspects as well: (i) SMYRF does not require (even though it might help) further pre-training before finetuning on downstream tasks. Therefore, SMYRF is a drop-in replacement of dense attention, while Longformer~\cite{beltagy2020longformer} requires some adaptation of the original dense attention. (ii) More importantly, the fixed sparsification idea used in Longformer \cite{beltagy2020longformer} is fundamentally different from our idea of using clustering to approximate attention and (iii) SMYRF can be used interchangeably with dense attention while Longformer cannot. As we showed, a trained SMYRF attention lower can be converted back to a normal dense attention layer during inference.

There are three research works that are very relevant to ours since they also propose data-driven attention within each group: (i) the Reformer \cite{reformer}, (ii) the Sparse Sinkhorn Attention~\cite{sinkhorn} paper and (iii) the Routing Transformer~\cite{routing_transformer}. Reformer~\cite{reformer} changes the dense attention layer twofold: (i) it tights vector representations of queries and keys, (ii) it sets their norm to be equal to $1$. Reformer is the first paper to propose LSH for clustering queries and keys. In Reformer, instead of using Asymmetric LSH, the authors use Angular distance LSH for clustering. This works because of (i), (ii), i.e. the Maximum Inner Product Search problem is equivalent to the Nearest Neighbor Search problem. We consider SMYRF as a generalized version of Reformer, since it employs Asymmetric LSH clustering to enable grouping of queries and keys that (i) do not have the same vectors, (ii) possibly live outside or inside the unitary $d-$dimensional disk. Apart from this, SMYRF and Reformer are similar: both networks sort vectors based on their LSH hash and both have linear attention complexity. Sinkhorn~\cite{sinkhorn} proposes a differentiable sorting module for clustering queries and keys. The sorting layer is trained end-to-end with the rest of the model. It has only been shown to work well for training from scratch and not for fine-tuning of pre-trained models. Routing Transformer~\cite{routing_transformer} proposes $k-$means clustering. In general, vectors that have small Euclidean distance are not guaranteed to have big inner product. To alleviate this, in Routing Transformer queries and keys are forced to have exactly the same vector representations and are also mapped to a $d-$dimensional unitary disk, exactly as Reformer proposed. Because of these changes,  also this method cannot be applied to pre-trained models. Routing transformer has some other weaknesses as well: (i) the complexity is $O(N^{1.5})$ instead of $O(N\log N)$ which is the attention complexity of SMYRF and Reformer and (ii) the clusters are not guaranteed to be balanced. To solve (ii), ~\cite{routing_transformer} proposes to keep the top-k vectors in each cluster. However, this is not guaranteed to work well since it depends on the clusters ordering.

Comparing to the aforementioned methods, SMYRF is the only method that assigns dynamically queries and keys in clusters and can be applied to pre-trained models. Due to its portability, SMYRF is the first sparse attention model to 
report GLUE results on par with the underlying models. As we showed, SMYRF can be used interchangeably with dense attention before, during and after training.
It also has linear attention complexity, similarly to Reformer. To the best of our knowledge, we are also the first to prove that the problem that all these methods are trying to solve is NP-hard. 

The optimization problem that SMYRF tries to solve is connected to the problem of bi-clustering~\cite{biclustering_first_paper}. Indeed, as shown in the proof of Theorem \ref{min_problem}, the goal in Attention Biclustering is to find a clustering of rows and columns of a matrix that maximizes the sum of the values of the clusters, where each value at position $(i, j)$ depends on the inner product of query $i$ and key $j$. For bi-clustering, iterative algorithms have been proposed~\cite{cheng2000biclustering}. Iterative techniques cannot be applied in the context of attention in which everything happens in a parallel fashion for fast execution in modern hardware.

Finally, there are a lot of others not attention related techniques that can be used to save memory and offer speedups. Examples of such techniques include knowledge distillation~\cite{hinton2015distilling, sanh2019distilbert}, reversible layers~\cite{gomez2017reversible}, gradient checkpointing~\cite{chen2016training}, quantization~\cite{hubara2016quantized} and pruning~\cite{80236, blalock2020state}. SMYRF and all these innovations are not mutually exclusive, i.e. they can be used together for maximum efficiency.

\section{Conclusions}
In this work we presented SMYRF, a novel type of balanced clustering to approximate attention. 
It is based on Asymmetric LSH with novel transformations and an adaptive clustering scheme. As it does not require changes to attention, SMYRF is the first sparse attention method that can be applied directly to pre-trained models. We showed powerful experimental results, in terms of performance, memory and speed.
We also defined the underlying optimization problem that SMYRF tries to solve and we proved it is NP-hard. The strong experimental performance of SMYRF inclines us to believe that good approximation algorithms exist for this problem. Proving approximation guarantees for our method and discovery of better approximation algorithms are left for future work.

\section{Acknowledgements}
We would like to wholeheartedly thank the TensorFlow Research Cloud (TFRC) program that gave us access to v3-8 Cloud TPUs and GCP credits that we used to run our Computer Vision experiments. 
This research has been supported by NSF Grants CCF 1763702,1934932, AF 1901292,
2008710, 2019844 research gifts by Western Digital, WNCG IAP, computing resources from TACC and the Archie Straiton Fellowship.
\newpage

\section{Broader Impact}
Our main contribution is to reduce the computational requirements for machine learning models with attention-layers.
Thus, any broader impact is likely to come from making these models more efficient in both memory impact and inference speed. 
We expect that this will be mostly a good thing since it democratizes the use of big attention layers: those who want to use such models but for whom the computational resources
required are too great (like university labs) will now have an easier time. 
Moreover, GANs and language models will become easier to deploy on phones or other embedded devices.
Further, more efficient training reduces the environmental and energy footprint of deep learning research. As the number of parameters of Transformer models grows, the latter becomes critical~\cite{strubell2019energy}.  

Negative consequences are also possible:
The idea of DeepFakes \cite{deepfakes} has been well-discussed elsewhere; a technique that makes these easier to create clearly has downsides.
On the other hand, any sufficiently determined actor (e.g. a nation-state attempting to commit election-fraud) 
already has access to such technology, so perhaps the marginal negative impact will not be that large.
Still, whenever computational requirements are reduced, the ease of taking
bad actions increases along with the ease of taking good actions.

Finally, the technique proposed in this paper relies heavily on the assumption that attention maps are approximately sparse.
It's possible (though we have no particular reason to think that this has happened or would happen) that, at some intermediate
layer of a complicated neural network, enforcing sparsity when the ground-truth maps are non-sparse could result
in ignoring salient features of atypical data points, thus resulting in fairness-related issues.
Determining whether these approximations cause fairness issues in general could
be an interesting subject for future work.

\clearpage

\title{Supplementary Material \\ SMYRF: \\ Efficient Attention using Asymmetric Clustering}

\author{%
  Giannis Daras\\
  Computer Science Department \\
    The University of Texas at Austin\\
  \texttt{giannisdaras@utexas.edu} \\
  \And
  Augustus Odena \\
  Google Research \\
  \texttt{augustusodena@google.com}
  \And
  Nikita Kitaev \\
  Google Research \\
  \texttt{kitaev@cs.berkeley.edu} 
  \And 
  Alexandros G. Dimakis \\
  ECE Department \\
  The University of Texas at Austin \\
  \texttt{dimakis@austin.utexas.edu} \\
}

% \begin{document}
\settitle

% \tableofcontents

\section{NP-hardness of Attention Biclustering}

To prove Theorem \ref{main_theorem}, we first prove the following lemma.
\begin{lemma}
The optimization problem:
$$
\min_{\mathcal C_t^L \in \mathcal C^L} ||\hat P_{\epsilon} - P ||_F^2 
$$ is NP-hard.
\label{maxmass}
\end{lemma}

\begin{proof}[Proof of Lemma \ref{maxmass}]
We will show that this problem is NP-hard, by showing that if we could solve in polynomial time all instances of this problem, we could solve in polynomial time the 3-dimensional matching problem (3-DM), which is known to be NP-complete. Following the notation of the main paper, we define $\epsilon=e^{-a}$ and $\hat P_\epsilon$ denotes the queries-keys product matrix with $-a$ in positions that correspond to queries and keys that do not belong to the same cluster.

It holds that:
$$
\min_{\mathcal C^L} ||\hat P_\epsilon - P||_F^2 = \min_{\mathcal C_t^L \in \mathcal C^L} \sum_{(q, k) \not \in \mathcal C_t^L} \left(q\cdot k - (-a)\right)^2
$$

$$
=\min_{\mathcal C_t^L \in \mathcal C^L} \left[ \sum_{(q, k) \in \mathcal Q \times \mathcal K} \left(q\cdot k + a \right)^2 - \sum_{(q, k) \in \mathcal C_t^L} \left(q\cdot k +a\right)^2\right] = 
\min_{\mathcal C_t^L \in \mathcal C^L} \left[- \sum_{(q, k) \in \mathcal C_t^L} \left(q\cdot k + a\right)^2\right]
$$

\begin{equation}
= \max_{\mathcal C_t^L \in \mathcal C^L} \sum_{(q, k) \in \mathcal C_t^L} \left(q \cdot k + a\right)^2
\end{equation}

Since for all given sets $\mathcal Q, \mathcal K$ we can create (in polynomial time) sets $\mathcal Q', \mathcal K'$ such that: $\left(q \cdot k + a\right)^2 = q'\cdot k', \quad \forall (q, k) \in \mathcal Q \times \mathcal K, (q', k') \in \mathcal Q' \times \mathcal K'$, the problem is equally hard to solving:

\begin{equation}
\max_{\mathcal C_t^L \in \mathcal C^L} \sum_{(q', k') \in \mathcal C_t^L} q'\cdot k' 
\label{maxmasseq}
\end{equation}

We can refer to this latest optimization problem as the max-mass problem.

Now consider the case where:
$|\mathcal Q|=|\mathcal K|, \quad L=|\mathcal Q|/2$, i.e. for this problem instance we have the same amount of queries and keys and we want to group them optimally to clusters with the constraint that each cluster should contain exactly $2$ queries and $2$ keys. 

Note that for 1 query and one key per cluster this becomes weighted bipartite matching (which is efficiently solvable). For 1 query and $m$ keys per cluster this is a generalized matching problem, which is also polynomially solvable~\cite{generalizedmatching}.

If we are able to solve the latter with a polynomial algorithm, then we can show that we can solve the 3-DM problem with a polynomial algorithm. 

Any instance of the 3-DM problem can be expressed with finite, disjoint sets $X, Y, Z$ and a set $T$ of triples $(x, y, z): \quad x \in X, y \in Y, z \in Z$. Visually, we can depict any instance of a 3-DM as a graph with three disjoint vertex sets, with $T$ containing the edges of the graph. For example, the 3-DM instance $X=\{1_{\textrm{red}}, 2_{\textrm{red}}\}, Y=\{1_{\textrm{blue}}, 2_{\textrm{blue}}\}, Z=\{1_{\textrm{green}}, 2_{\textrm{green}}\}, T = \{ (1_{\textrm{red}}, 1_{\textrm{blue}}, 1_{\textrm{green}}), (1_{\textrm{red}}, 2_{\textrm{blue}}, 2_{\textrm{green}}), (2_{\textrm{red}}, 1_{\textrm{blue}}, 1_{\textrm{green}})\}$ is shown in (1,1) of Figure \ref{np-hardness}. We are looking for a set $T' \subseteq T$ in which every vertex is covered exactly once. Finding this solution, in case it exists, it is known to be an NP-hard problem. For this example, there is a valid solution, which is shown in (1, 2) of Figure \ref{np-hardness}.

We can transform any instance in the following way: we create one query and one key vector for each vertex $x \in X$ with the property that their inner product is some large positive constant $r_1\in \mathbb R^{+}$. We can visualize this using red edges, following the previous example where we denoted with red color the vertices of $X$. We also set the inner product of any key vector that corresponds to vertex of $X$ with all the other query vectors to be $0$. Visually, a ``missing" edge means that the inner product of the corresponding vectors is $0$ (no-reward). We also create a key vector for each vertex $y \in Y$ with the property that if $(x, y, z) \in T$ for some $z$, then the key vector for $y$ and the query vector for $x$ have inner product $r_1$, else $0$. We can show the non-zero edges of this category visually with blue color, following the previous example. Note that blue and red edges are equivalent in terms of the inner product between the vertices they connect, since both have inner product $r_1$. Finally, we create a query vector for each $z \in Z$ with the property that if $(x, y, z) \in T$ for some $x$ then the key vector for $y$ and the query vector for $z$ have inner product $r_2$, else $0$ where $r_2 \in \mathbb R^{+}$ is a small positive constant. Again, we can show the non-zero edges of this category with green color, following the previous example. For the given example, the transformation is shown in (2, 1) of Figure \ref{np-hardness}.

We have hypothesized that we have a polynomial algorithm to solve the max-mass problem of (\ref{maxmasseq}). The key observation for our proof is that, by construction, the best cluster in terms of potential accumulated mass is a cluster with one red, one blue and one green edge, as the ones shown right of the dashed bar of Figure \ref{np-hardness-cases}. Indeed, the only way to obtain a cluster of more mass is to group two blue vertices with two red vertices, as shown in (1, 1) of Figure \ref{np-hardness-cases}. By doing that, you earn one more $r_1$ compared to the clustering shown in (1, 2) of Figure \ref{np-hardness-cases}, but you lose $2\cdot r_1$, which are the rewards that they red keys could give (as they are left with no connections). Thus, the two clusterings on the right side of Figure \ref{np-hardness-cases} are preferable compared to any other potential two clustering that can be obtained by choosing the left grouping.

Since we have proved that the best possible clustering is one with one red, one blue and one green edge, it is now left to prove that if there is a 3-DM, then it is possible to group all queries and keys into clusters with this optimality property. Indeed, if there is a 3-DM, we can cover each vertex exactly one time, by matching any vertex of $X$ with a vertex from $Y$ and a vertex from $Z$. With our transformation, this means that we can group each red node with itself and one blue and one green vertex, which is an optimal cluster as it contains one red, one blue and one green edge. Thus, solving polynomially our problem would mean that we could also solve in polynomial time the 3-DM, which is known to be NP-hard.

\begin{figure}[!htp]
    \centering
    \includegraphics[width=0.65\textwidth]{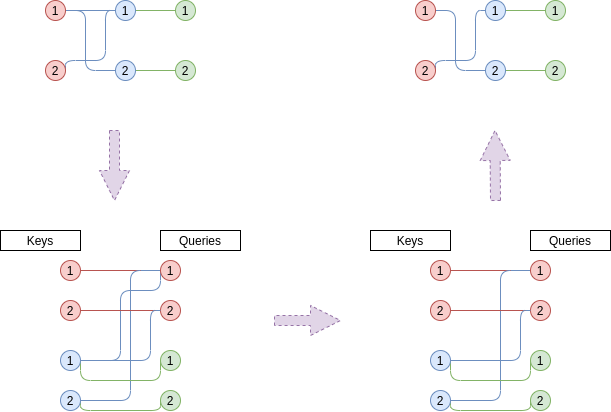}
    \caption{(1, 1): Instance of $3-$DM. We denote with red color the $X$ vertex set, with blue the $Y$ vertex set and with green the $Z$ vertex set. (2, 1): Ours transformation for the reduction. Red and blue edges have reward $r_1$, while green edges have reward $r_2<<r_1$. Missing edges have reward $0$. We create one query and one key for each vertex of $X$. We also create one key (blue color) for each vertex of $Y$ and one query (green color) for each vertex of $Z$. Connections between red queries and blue keys, as well as, connections between blue keys and green queries follow the problem instance. (2, 2): Optimal queries, keys clustering in groups of 2 for the max-mass \ref{maxmasseq} problem. (1, 2): Transformation of (2, 2) solution back to the $3-$DM instance.}
    \label{np-hardness}
\end{figure}

\begin{figure}[!htp]
    \centering
    \includegraphics[width=0.6\textwidth]{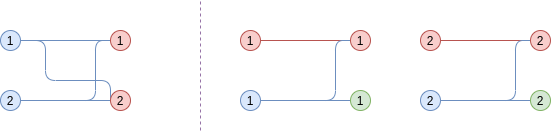}
    \caption{Illustration of potential clusterings. (1, 1): sub-optimal clustering. (1, 2): optimal clusterings. Even though the clustering at the left side obtains more mass compared to any of the clusterings in the right side, it loses entirely the rewards that red keys can give. Indeed, clustering on the left side has one $r_1$ reward more than any of the two clusterings on the right, but in further clusterings red keys $\{1, 2\}$ will not be matched with anything (by construction) and thus a total reward of $2r_1$ will be lost.}
    \label{np-hardness-cases}
\end{figure}
\end{proof}

\begin{proof}[Proof of Theorem \ref{main_theorem}]
We will show that if we can solve in polynomial time the problem: $\min_{\mathcal C^L} ||\sigma(\hat P_0) - \sigma(P)||_F^2$, then we can also solve in polynomial time the problem $\min_{\mathcal C^L}||\hat P_\epsilon - P||_F^2$ (for an appropriate $\epsilon$) which we have proven to be NP-hard. 

We are given sets $\mathcal Q, \mathcal K$ and a number $L$. For each $q_i \in \mathcal Q$, we create a key vector $k_{q_i}$ such as $q_j \cdot k_{q_i} = \begin{cases}
a, \quad \textrm{if} \ i=j \\
-\infty, \quad \textrm{o/w}
\end{cases}$, where $a$ is a positive constant the choice of which we will determine later in this proof.

We denote the augmented key set with $\mathcal K'$. 

We will now solve, with our hypothetical polynomial algorithm, the following optimization problem for our new input set:

$$\min_{\mathcal C^L} ||\sigma(\hat P_0) - \sigma(P)||_F^2
$$

It holds that:
$$
\min_{\mathcal C^L} || \sigma(\hat P_0) - \sigma(P) ||_F^2 = \min_{\mathcal C_t^L \in \mathcal C^L}\sum_{(q, k) \not \in \mathcal C_t^L}\left(\frac{e^{q\cdot k}}{q_D}\right)^2 + \sum_{(q, k) \in \mathcal C_t^L} \left( \frac{e^{q\cdot k}}{q_D} - \frac{e^{q\cdot k}}{q_{\mathcal C_t^L}}\right)^2
$$

$$
= \max_{\mathcal C_t^L \in \mathcal C^L}\sum_{(q, k) \in \mathcal C_t^L} e^{2q\cdot k} \cdot \left(\frac{2}{q_D \cdot q_{\mathcal C_t^L}} - \frac{1}{q_{\mathcal C_t^L}^2} \right),
$$
where $q_D$ denotes the denominator of the dense softmax and $q_{C_t^L}$ denotes the denominator of the cluster softmax, i.e. $q_D = \sum_{k \in \mathcal K} e^{q \cdot k}$ and $q_{C_t^L} = \sum_{k \in \mathcal C_t^L} e^{q\cdot k}$ for a given cluster $\mathcal C_t^L \in \mathcal C^L$.

We will now show that for a proper choice of $a$, this problem is equivalent to:

$$
\max_{\mathcal C_t^L \in \mathcal C} \sum_{(q,k) \in \mathcal C_t^L} e^{2q\cdot k}.
$$

Let $R_q =\frac{2}{q_D q_{C_t^L}} - \frac{1}{q^2_{C_t^L}}$. As we increase the value of $a$, the inner product of its query with its' special key gets significantly bigger compared to other inner products and thus for large enough values of $a$, we know that each query will get clustered with its' special key. We can control how close $q_D, q_{\mathcal C_t^L}$ are by setting appropriately the $a$ value. Specifically, we choose $a$ such that $q_D (1-\epsilon) < q_{\mathcal C_t^L}, \ \forall q \in \mathcal Q, \ \mathcal C_t^L \in \mathcal C^L$, where $\epsilon=\epsilon(a)$ a small positive constant the choice of which we will determine soon. By definition, $q_{\mathcal C_t^L}$ is always smaller than $q_D$, and thus we for that choice of $a$ we have $q_D(1-\epsilon) < q_{\mathcal C_t^L} < q_D$. Then, $R_q > \frac{2}{q_D^2} - \frac{1}{q_D^2 (1-\epsilon)^2} = \frac{1}{q_D^2} (2 - \frac{1}{(1-\epsilon)^2}) = \frac{1-\epsilon'}{q_D^2}$ where $1 + \epsilon' = \frac{1}{(1-\epsilon)^2}$. But also, $R_q = \frac{2q_{\mathcal C_t^L} - q_D}{q_{\mathcal C_t^L}^2q_D} < \frac{2q_D - q_D}{q_{\mathcal C_t^L}^2q_D} = \frac{1}{q_{\mathcal C_t^L}^2} < \frac{1}{(1-\epsilon)^2q_D} = (1 + \epsilon')\frac{1}{q_D^2}$. Then, we have that:
\begin{equation}
\frac{1-\epsilon'}{q_D^2} < R_q < \frac{1 + \epsilon'}{q_D^2}.
\label{boundR}
\end{equation}

Now consider the following optimization problems:
$$
\begin{cases} 
P_0: \quad \max \sum_{(q, k) \in \mathcal C_t^L} e^{2qk} R_q \\
P_1: \quad \max \sum_{(q, k) \in \mathcal C_t^L} \frac{e^{2qk}}{q_D}
\end{cases}.
$$
Let $F(c), G(c)$ the objective functions of $P_0, P_1$ respectively.

Using (\ref{boundR}), we get that:
\begin{equation}
(1-\epsilon') G(c) \leq F(c) \leq (1 + \epsilon ') F(c).
\label{boundGF}
\end{equation}

Our claim is that for a suitable choice of $\epsilon'$, i.e. for a suitable choice of $a$, it holds that $\argmax P_0 = \argmax P_1$\footnote{We assume that if there is a set of optimal solutions, then we pick with the same order from that set for both problems.}.We prove that by contradiction. Let $c_1$ be the optimal choice of $P_0$ and $c_2$ be the optimal choice of $P_1$. Then, we know that $F(c_1) > F(c_2)$ and $G(c_2) > G(c_1)$. Using (\ref{boundGF}), we get that:
\begin{equation}
(1-\epsilon') G(c_1) - (1 + \epsilon')G(c_2) < F(c_1) - F(c_2) < (1+\epsilon')G(c_1) - (1-\epsilon') G(c_2).
\label{solbound}
\end{equation}
We denote with $d$ the gap between the optimal value $F(c_1)$ and the non optimal solution $F(c_2)$, i.e. $d = F(c_1) - F(c_2)$. Then, from (\ref{solbound}), we get that:
$$
d < (1 + \epsilon')G(c_1) - (1-\epsilon')G(c_2) - (1-\epsilon')G(c_1) + (1+\epsilon')G(c_2) = 2e' (G(c_1) + G(c_2)). 
$$

Let $\theta_1$ the maximum value of $G(c_1) + G(c_2)$ among all the clusterings $c_1, c_2 \in C^L$, i.e. among all the possible valid clusterings in $L$ groups. Then, $d < 2\epsilon' \theta_1$. However, since $F$ is a function that maps from discrete clusterings to real numbers, two non-optimal solutions of $F(c)$ differ for at least a minimum distance. In that case, the minimum distance should be at least $e^{p_{\min}}R_{\min}$, where $p_{\min}$ is the minimum product between any query and any key and $R_{\min}$ is the minimum value that $R$ can take for any clustering. Let $\theta_2 = e^{p_{\min}}R_{\min}$. Then, $d \geq \theta_2$. If we choose $\epsilon'$ such that: $2\epsilon' \theta_1 < \theta_2$ then we have a contradiction. This is always possible since we can set the value of $\epsilon'$ to arbitrarily small values as we grow $a$ arbitrarily big. Thus, we proved that the problems $P_0, P_1$ have the same $\argmax$ for a proper choice of $a$. Then, for that choice of $a$ the problem $\min_{C^L} ||\sigma(\hat P_0) - \sigma(P)||_F^2$ is equivalent to $P_1$ which is equivalent to the problem:
$$
\max_{\mathcal C_t^L \in \mathcal C} \sum_{(q,k) \in \mathcal C_t^L} e^{2q\cdot k},
$$
since $q_D$ does not affect the choice of optimal clusters.

In the latter problem, we can replace all queries $q$ and keys $k$ with new vectors $q', k'$ such that: $q'\cdot k' = e^{2q\cdot k}$. This is equally hard to solving:
$$
\max_{\mathcal C_t^L \in \mathcal C} \sum_{(q,k) \in \mathcal C_t^L}  q\cdot k
$$
which we proved to be NP-hard.
\end{proof}

\section{Code}
To encourage further research in sparse attention models, we open-source all our code and we release a Python package, named \texttt{smyrf}. The repository for the code is the following: \href{https://github.com/giannisdaras/smyrf}{https://github.com/giannisdaras/smyrf} . \texttt{smyrf} implements SMYRF attention for Pytorch~\cite{paszke2019pytorch}. We plan to release implementation for Tensorflow~\cite{tensorflow} soon as well. \texttt{smyrf} contains various examples on pre-training and finetuning state-of-the-art models for Computer Vision and Natural Language Processing tasks. Regarding examples, at the moment \texttt{smyrf} includes:

\begin{itemize}
    \item a TPU-compatible implementation of SMYRF-BigGAN, based on the official Pytorch implementation (\href{https://github.com/ajbrock/BigGAN-PyTorch}{https://github.com/ajbrock/BigGAN-PyTorch}) for GPUs.
    \item code for training SMYRF-BigGAN on Celeba-HQ on a single TPU device.
    \item interactive notebooks showing how to use a pre-trained BigGAN for image generation with SMYRF on Celeba-HQ and ImageNet.
    \item tools to visualize cluster memberships for pixels of SMYRF generated images. 
    \item code for replacing dense attention with SMYRF layers for state-of-the-art pre-trained NLP models, compatible with HuggingFace's Transformers~\cite{wolf2019huggingfaces} library.
    \item interactive notebooks for fine-tuning pre-trained NLP models on GLUE~\cite{GLUE} and IMDB~\cite{IMDB}.
    \item tools for profiling SMYRF's performance compared to dense attention. 
\end{itemize}

We also share the weights of SMYRF-BigGAN trained on Celeba-HQ at resolutions $128\times 128$ and at $256\times 256$ with attention at $128\times 128, 256\times 256$ respectively. Although these models are outperformed by non-attention GANs (e.g. StyleGAN~\cite{stylegan2, stylegan1}), we believe that releasing them will help researchers understand better attention at higher resolutions. Hopefully, SMYRF will motivate the usage of more attention layers
on new GAN architectures.

\section{Singular values decay for pre-trained models}
As noted in the paper, row-wise softmax 
can change the rank of a matrix. For example, the matrix $\begin{bmatrix} 1 & 0 \\ 2 & 0 \end{bmatrix}$ has rank $1$, while the matrix $\sigma\left(\begin{bmatrix} 1 & 0 \\ 2 & 0 \end{bmatrix}\right) = \begin{bmatrix}
0.7311 & 0.2689 \\
0.8808 & 0.1192
\end{bmatrix}$ has rank $2$. Back to the context of attention, we have defined the product matrix $P = Q\cdot K^T$, where $Q:\mathbb R^{|\mathcal Q| \times d }$ represents the queries matrix and $K:\mathbb R^{|\mathcal K| \times d}$ the keys matrix.
By the definition of rank, if the embeddings dimension is smaller than the sequence length dimension, i.e. $d < \min(|\mathcal Q|, |\mathcal K|)$, then P is low rank. However, the attention matrix after softmax, i.e. $\sigma(P)$, could be a full rank matrix. In this section, we provide experimental evidence that 
attention maps produced by pre-trained models are actually near low-rank. 

Figures \ref{sing_decay_biggan}, \ref{sing_decay_bert} depict the singular values of the attention maps (for a random input\footnote{We experimented with different random inputs and there is no qualitative difference in the decay of singular values)}) for a pre-trained BigGAN (attention map dimensions: $4096\times 1024$) and a pre-trained BERT (shown attention map dimensions: $64\times 64$, $256\times 256$). For the pre-trained BigGAN (Figure \ref{sing_decay_biggan}) the decay in singular values is exponential. Specifically, in Figure \ref{sing_decay_biggan} most singular values are very close to 0, which means that the attention map is effectively low rank. Figure \ref{sing_decay_bert} shows decay of singular values for a pre-trained BERT for sequence lengths: (a) 64, (b) 128. We illustrate decay for 144 heads (12 heads for each one of the 12 layers). For the majority of heads, singular values decay exponentially. We also see that the heads that do not demonstrate exponential decay in the singular values maintain this property for both inputs (e.g. see the red line in both plots). In our experiments, we find that these heads are harder to approximate with SMYRF.

\begin{figure}[!htp]
    \centering
    \includegraphics[width=0.75\textwidth]{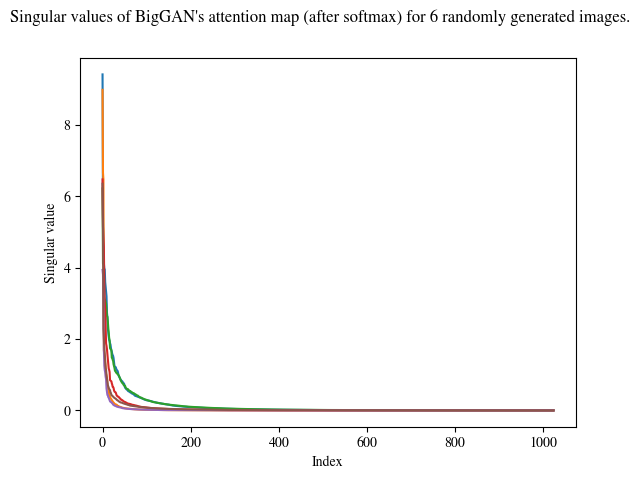}
    \caption{Decay of singular values of the attention map (after softmax) of a pre-trained BigGAN. Decay of singular values is exponential, which means that the matrix after softmax is effectively low rank.}
    \label{sing_decay_biggan}
\end{figure}
\begin{figure}[!htp]
    \begin{subfigure}{\textwidth}
    \centering
        \includegraphics[width=0.75\textwidth]{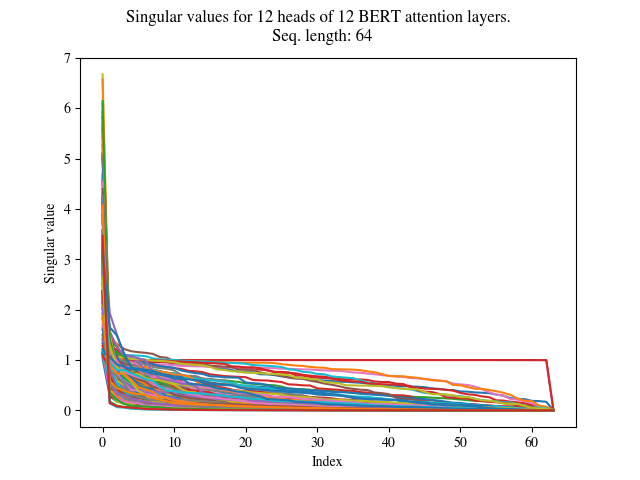}
        \caption{}
        \label{sing_decay_bert64}
    \end{subfigure}
    
    \begin{subfigure}{\textwidth}
    \centering
    \includegraphics[width=0.75\textwidth]{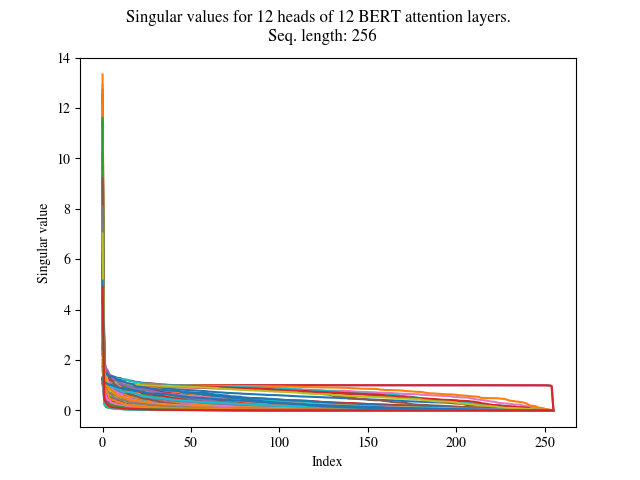}
    \caption{}
    \label{sing_decay_bert128}
\end{subfigure}
\caption{Decay of singular values for a pre-trained BERT for sequence lengths: (a) 64, (b) 128. We show decay of singular values 144 heads (12 heads for each one of the 12 layers). For the majority of heads, singular values decay exponentially. We also see that the heads that do not demonstrate exponential decay in the singular values maintain this property for both inputs (e.g. see the red line in both plots). We find that these heads are harder to approximate with SMYRF.}
\label{sing_decay_bert}
\end{figure}

\section{Cluster memberships for generated images of a pre-trained BigGAN}
In this section, we visualize how SMYRF's adaptive clustering algorithm assigns queries in clusters for a pre-trained BigGAN. This inspection gives useful insights into how the algorithm actually works in practice. 

Top row of Figure~\ref{vismemberships} shows a random maltese dog generated by a pre-trained BigGAN~\cite{biggan}. The second row, illustrates how a single SMYRF hashing round assigns queries and keys for this particular image in two clusters: the first cluster is denoted with gray and the second with white color. As shown, SMYRF assignments preserve locality while enabling the modeling of arbitrary complex dependencies between input pixels. Indeed, pixels in the same neighborhoods are mostly organized in the same cluster. This observation is even more pronounced for background pixels (see big gray blocks). However, we also see that distant pixels sometime belong to the cluster as well. By only looking at the assignments in clusters (second row), we can infer that the image is roughly separated in three parts: the top part (mostly gray pixels), the middle part (mostly white pixels) and the bottom part (mostly gray pixels). These parts correspond to the top background, the dogs' face and the bottom background respectively. Third row of Figure~\ref{vismemberships} illustrates (for the same image) assignments in 128 clusters. Each cluster contains 32 queries and is denoted with a distinct color. Again, we observe that clusters are often local. Indeed, usually consecutive pixels or nearly consecutive pixels are denoted with the same color. For such large number of clusters, it becomes very difficult to extract semantic information from the clustering map without looking at the original image. However, by careful looking at both the attention map and the generated image we can make interesting observations. For instance, we see that distant background pixels are clustered together with much greater frequency compared to other distant non-background pixels. In other words, SMYRF often clusters together background pixels even if they belong to distant grid positions in the generated image (see for example colors in top and last row of the grid). 

\begin{figure}[!htp]
\centering
\begin{subfigure}{\textwidth}
\centering
\includegraphics[]{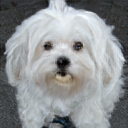}
\caption{Generated maltese dog from a pre-trained BigGAN.}
\label{sample_memberships}
\end{subfigure}
\bigbreak
\begin{subfigure}{0.6\textwidth}
\includegraphics[width=\textwidth]{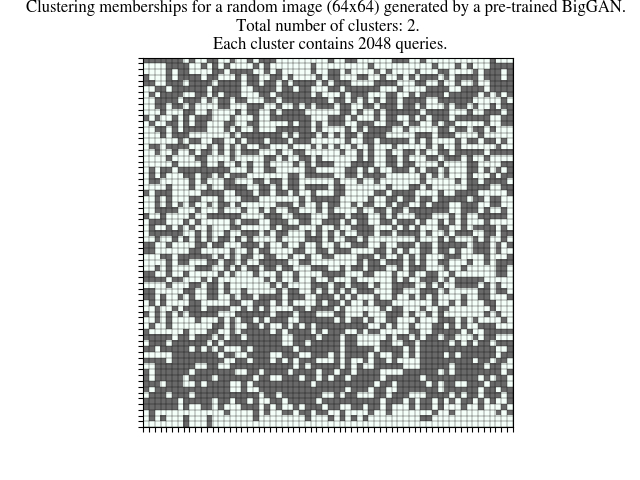}
\caption{Visualization of SMYRF cluster assignments for this image (single hash). Total number of clusters: 2.}
\label{2c}
\end{subfigure}
\bigbreak 
\begin{subfigure}{0.6\textwidth}
\includegraphics[width=\textwidth]{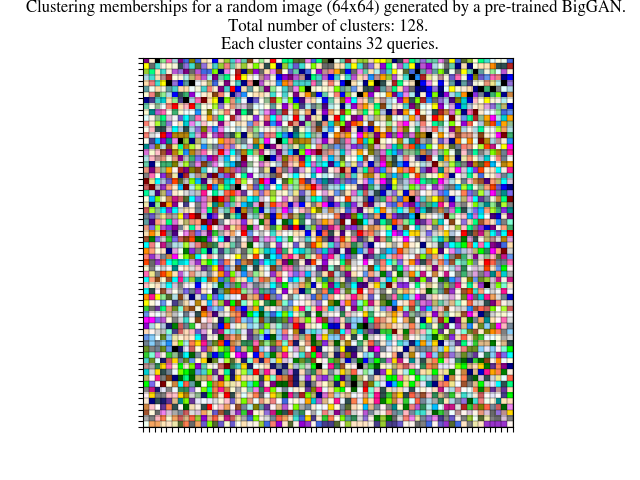}
\caption{Visualization of SMYRF cluster assignments for this image (single hash). Total number of clusters: 128.}
\label{128c}
\end{subfigure}
\caption{Visualization of clustering assignments for a generated image by a pre-trained BigGAN.}
\label{vismemberships}
\end{figure}

\section{SMYRF Clustering}
\subsection{Asymmetric Locality Sensitive Hashing (ALSH)}
SMYRF clusters depend on the hashing indices of asymmetrically transformed queries and keys. As mentioned in the paper, we are looking for functions $F: \mathbb R^{d}\to \mathbb R^{d'}, G: \mathbb R^d \to \mathbb R^{d'}$ such as: $||F(q) - G(k)||_2^2 = D(q\cdot k), \ \forall (q, k)$ where $D:\mathbb R \to \mathbb R$ a decreasing function that depends only on the inner product $q\cdot k$. Essentially, functions $F, G$ are applied to queries and keys to convert the problem of Maximum Inner Product Search (MIPS) to Nearest Neighbor Search (NNS). For the latter problem, a lot lot of effective Locality Sensitive Hashing (LSH) functions have been proposed (e.g. ~\cite{e2lsh, VoronoiLSH, nearoptimallsh}). The novel idea of converting MIPS to NNS is called Asymmetric Locality Sensitive Hashing (ALSH) and was first introduced in~\cite{l2lsh}. Since then, a lot of different asymmetric transformations have been proposed~\cite{h2lsh, e2lsh, xbox}. In this section, we show why previously proposed transformations are not suitable for our problem and how our novel asymmetric transformations, defined in Equation 4, relate to previous work. 

We list the asymmetric transformations that have been widely used to convert a MIPS to NNS:

\[
\left\lbrace
\begin{tabular}{@{}p{15cm}p{15cm}} 
\cite{l2lsh}: $F(q_i) = \left[q_i; \frac{1}{2}, ...; \frac{1}{2}\right], \ G(k_i) = \left[ Uk_i; ||Uk_i||_2^2; ...;||Uk_i||_2^{2^m}\right]$ \\ \\

\cite{xbox}: $F(q_i) = \left[q_i ; 0\right], \ G(k_i) = \left[k_i; \sqrt{M_K^2 - ||k_i||_2^2}\right]$ \\ \\

\cite{h2lsh}: $F(q) = \frac{M_K}{||q||_2}\cdot \left[ q;0\right], \ G(k) = \left[k; \sqrt{M_K^2 - ||k||_2^2}\right]$ \\
\end{tabular}\right.
\]
where $M_K = \max_{k}||k||_2$ and U a positive constant such as: $||U\cdot k_i||_2^{2^{m + 1}} \to 0, \ \forall k_i \in \mathcal K$. 
The corresponding Euclidean distances of the transformed vectors are given below:
\[
\left\lbrace
\begin{tabular}{@{}p{15cm}p{15cm}} 
\cite{l2lsh}:
$||F(q_i) - G(k_i)||_2^2 = ||q_i||_2^2 + \frac{m}{4}  - 2Uq_i\cdot k_i + ||U \cdot k_i||_2^{2^{m+1}}$\\ \\

\cite{xbox}:  
$ ||F(q_i) - G(k_i)||_2^2 = ||q_i||_2^2 + M_K^2 - 2 q_i \cdot k_i$ \\ \\

\cite{h2lsh}:
$||F(q_i) - G(k_i)||_2^2 = 2\cdot M_K^2 -2\frac{M_K}{||q_i||}\cdot q_i\cdot k$
\end{tabular}\right.
\]

In all these transformations the Euclidean distance of the transformed vectors, i.e. $||F(q_i) - G(k_i)||_2$ decreases linearly with the inner product $q_i\cdot k_i$. However, an extra term, $p(||q_i||)$, appears. Indeed, these transformations were proposed for the case of a single query (e.g. a user) and multiple keys (e.g. movies) and for such applications $||q_i||$ is considered constant. On the contrary, for our setting, the transformations of \cite{l2lsh, h2lsh, xbox} cannot be applied since $||q||_2$ is no longer a constant. To illustrate this better, consider the case where $q_1, q_2 \in \mathcal Q$ with $q_1 \neq q_2$ and $k \in \mathcal K$ a key such as: $q_1 \cdot k = q_2 \cdot k$. Since we are looking for big inner products, we expect to have transformations $F, Q: ||F(q_1) - G(k) ||_2 = ||F(q_2) - G(k)||_2$. For \cite{l2lsh, xbox}, if $||q_1||_2 < ||q_2||_2$ then $||F(q_1) - G(k)||_2 < ||F(q_2) - G(k)||_2$ and for \cite{h2lsh}: $||F(q_1) - G(k)||_2 > ||F(q_2) - G(k)||_2$. Thus, all \cite{l2lsh, h2lsh, xbox} do not satisfy our desired property, i.e. $||F(q_1) - G(k)||_2 = ||F(q_2) - G(k)||_2$. To solve this problem, we propose (see main paper) the novel asymmetric functions:
\begin{equation}
    F(q_i) = \left[q_i; 0; \sqrt{M_Q^2 + M_K^2 - ||q_i||_2^2} \right], \qquad G(k_i) = \left[k_i; \sqrt{M_Q^2 + M_K^2 - ||k_i||_2^2}; 0\right]
\end{equation}
where we use the constants $M_Q = \max_{q_i}||q_i||_2, \quad M_K = \max_{k_i}||k_i||_2$, or any other upper bound on the norms. 
With this transformation, all queries and keys are mapped to a $(d+2)$-dimensional ball with radius $\sqrt{M_Q^2 + M_K^2}$ and the distance of the transformed vectors decreases linearly with the inner product of the original vectors:
\begin{equation}
    ||F(q_i) - G(k_i)||_2^2 = 2 \cdot \left( M_Q^2 + M_K^2 -  q_i\cdot k_i\right).
\end{equation}
Note that the Euclidean distance of the transformed vectors depends only on the inner product of the original vectors and not on individual norms $||q_i||_2$ as in previous work.

\subsection{Adaptive Clustering}
The next step, after the asymmetric transformations, is to map the transformed queries $F(q)$ and keys $G(k)$ to real numbers, so that if $||F(q) - G(k)||_2$ is small, then $|h(F(q)) - h(G(k))|$ is also small with high probability, where $h:\mathbb R^{d'} \to \mathbb R$ is the mapping function. After mapping, we sort independently queries and keys based on their hash and we split them into groups of equal size.  
There are numerous hashing functions ~\cite{andoni2015practical, e2lsh, VoronoiLSH, qlsh} $h:\mathbb R^{d'} \to \mathbb R$ that belong to the LSH family that we can leverage to achieve that. One of the most widely adopted hash functions for locality sensitive hashing is E2LSH~\cite{e2lsh}: 
\begin{equation}
h_{\textrm{E2LSH}}(u) =  \left \lfloor \frac{(u \cdot a) + b}{r} \right \rfloor
\label{e2lsh_hash}
\end{equation}
where $a=(a_1, ..., a_d') \in \mathbb R^{d'}$ with $a_i \in \mathcal N(0, 1)$ and $b \in \mathcal U(0, r)$ and $r$ is a scalar parameter which controls LSH sensitivity. Since we re-group vectors by sorting on their LSH index, the floor operator and the division with $r$ are not needed. 
Our simplified hashing function is defined as:
\begin{equation}
    h_{\textrm{ours}}(u) = (u\cdot a) + b
    \label{simplified_e2lsh}
\end{equation}
We roughly removed a division by a constant. Thus, this simplified hashing function preserves the locality-sensitive properties of E2LSH~\cite{e2lsh}. Namely, if $||u_1 - v_1||_2 \leq || u_2 - v_2||_2$ then with high probability: $|h(u_1) - h(v_1)| \leq |h(u_2) - h(v_2)|, \ \forall u_1, u_2, v_1, v_2 \in \mathbb R^{d'}$. 

\subsection{Merging hashing rounds}
In our experiments, we run multiple hashing rounds each time, similarly to \cite{reformer}. Each time we run LSH, we end up with a (possibly) different clustering assignment and thus (possibly) different attention output. Specifically, we repeat the process $H$ times (where $H$ is usually a small constant, e.g. 8) to reduce the probability that we miss big inner products.
In this section, we explain how we merge the partial attention outputs (made from different hashing rounds) into a single attention output.

Without loss of generality, we will present the merging algorithm for a single query $q$. At each clustering round $h$ we get (from the adaptive clustering) a set of key vectors $\mathcal K_{h_q} \subseteq \mathcal K$. The corresponding attention output is:
$$
o_q^h = \sum_{k \in \mathcal K_{h_q}} w_k v_k, \qquad w_k = \frac{e^{q \cdot k}}{\sum_{k' \in \mathcal K_{h_q}} e^{q \cdot k'}} 
$$
We merge the attention outputs of the different rounds with a weighted sum. The weight, $a_h$, for each round $h$, is the fraction of the softmax mass that was acquired in this round to the total mass acquired by all rounds. Formally the attention output $o_q'$ for query $q$ is computed as:
\begin{equation}
o_q' = \sum_{h=1}^{H}a_h \cdot \sum_{k \in \mathcal K_{h_q}} w_k v_k, \qquad w_k =  \frac{e^{q \cdot k}}{\sum_{k' \in \mathcal K_{h_q}} e^{q \cdot k'}}, \qquad a_h = \frac{  \sum_{k' \in \mathcal K_{h_q}} e^{q \cdot k'}}{\sum_{n=1}^{H}\sum_{k' \in \mathcal K_{n_q}} e^{q \cdot k'}}
\label{merging_scheme}
\end{equation}

To explain this merging scheme, we will show that under certain assumptions, this merging scheme can lead to exact approximation of the real attention output. We start by listing these assumptions. 
\begin{assumption}[Sparsity of weights]
For any given query $q \in \mathcal Q$, the key set $\mathcal K$ has at most $T$ and at least one vectors $k_i \in \mathcal K_q$ such as:
$$
k_i \in \mathcal K_q, \ k_j \not \in \mathcal K_q \Rightarrow \frac{e^{q \cdot k_j}}{e^{q \cdot k_i}} = 0
$$
\label{sparsity}
\end{assumption}
From Assumption \ref{sparsity}, it follows that at most $T$ and at least one key vector $k_i$ gets a non-zero score, $w_i \neq 0$, after softmax.

\begin{assumption}[Fairness of LSH clustering]
For any given query $q \in \mathcal Q$ and two keys $k_1, k_2 \in \mathcal K$, if $w_{k_1} \neq 0 \ \wedge w_{k_2} \neq 0$, then $\sum_{n=1}^{H} \sum_{k_1 \in \mathcal K_{n_q}} 1 = \sum_{n=1}^{H} \sum_{k_2 \in \mathcal K_{n_q}} 1$
\label{lsh_fairness}
where $H$ denotes the hashing rounds and $\mathcal K_{n_q}$ denotes the chosen key set for query $q$ at hash round $n$.
\end{assumption}

Assumption \ref{lsh_fairness} simply states that each query is clustered the same number of times with all its' big inner products along the different hashing rounds.

\begin{assumption}[Effectiveness of LSH clustering]
There is a small constant $H$, which denotes the number of hashing rounds, such as:
$$
\forall k \in \mathcal K: \ w_q \neq 0 \Rightarrow \exists n: \ 1 \leq n \leq H \ \wedge \ \ k \in \mathcal K_{q_n}.
$$
\label{good_clustering}
\end{assumption}
The latter assumption states that we need a small number of hashing rounds $H$ to catch all big inner products of a given query.

We state the following theorem:
\begin{theorem}
If Assumptions \ref{sparsity}, \ref{lsh_fairness}, \ref{good_clustering} hold, then our approximation algorithm is exact. 
\label{exactness}
\end{theorem}

\begin{proof}[Proof of Theorem \ref{exactness}]
With our merging scheme (Equation \ref{merging_scheme}), the attention output is:
\begin{equation}
o_q' = \sum_{h=1}^{H}\sum_{k\in \mathcal K_{h_q} } \left( \frac{  \sum_{k' \in \mathcal K_{h_q}} e^{q \cdot k'}}{\sum_{n=1}^{H}\sum_{k' \in \mathcal K_{n_q}} e^{q \cdot k'}} \cdot \frac{e^{q \cdot k}}{\sum_{k' \in \mathcal K_{h_q}} e^{q \cdot k'}} \right) \cdot v_k = \sum_{h=1}^{H} \frac{\sum_{k\in \mathcal K_{h_q} } e^{q\cdot k}\cdot v_k }{\sum_{n=1}^{H}\sum_{k' \in \mathcal K_{n_q}} e^{q \cdot k'}}
\end{equation}

Under Assumption \ref{sparsity}, the dense attention output for this query is the vector:

$$
o_q = \sum_{k \in \mathcal K_q} \frac{e^{q \cdot k}}{\sum_{k' \in \mathcal K_q} e^{q \cdot k'}} \cdot v_k 
$$
where $K_q$ is the set of keys $k_i$ for query $q$ for which $w_i \neq 0$.

Under Assumption \ref{good_clustering}, all keys that have big inner product with a given query $q$ are clustered with that query, at least one time. Also, under Assumption \ref{lsh_fairness}, all these keys are clustered the same amount of times with each query. We will denote the amount of a query is clustered with each one of its' big inner products with $N_q$. It holds that:
\begin{equation}
\sum_{n=1}^{H}\sum_{k' \in \mathcal K_{n_q}} e^{q \cdot k'} = N_q \cdot \sum_{k' \in \mathcal K_q} e^{q\cdot k'}
\label{exact_denom}
\end{equation}

By substitution in Equation \ref{exact_denom}, we get:
\begin{equation}
o_q = \frac{\sum_{n=1}^{H} \sum_{k\in \mathcal K_{h_q}}e^{q\cdot k}\cdot v_q}{N_q \cdot \sum_{k' \in \mathcal K_q} e^{q\cdot k'}}
\label{before_exact}
\end{equation}

Under Assumptions \ref{sparsity}, \ref{lsh_fairness} small inner products get a zero-score and all big inner products are clustered $N_q$ times each. Thus, we can write for the nominator: $\sum_{n=1}^{H} \sum_{k\in \mathcal K_{h_q}}e^{q\cdot k}\cdot v_q = N_q \sum_{k \in \mathcal K_q} e^{q\cdot k'}$.

Substituting to Equation \ref{before_exact}, we get:
$$
o_q' = \sum_{k \in K_q}\frac{e^{q \cdot k}}{\sum_{k' \in \mathcal K_q} e^{q \cdot k'}} v_k = o_q
$$
\end{proof}

In this section, we explained in detail our merging scheme. We also showed that under certain assumptions on the data, this scheme leads to exact approximations of dense attention output. We fully understand that the assumptions are far too tight to hold in practice and since distortion is introduced. However, as we demonstrated in the Experiments section, the distortion is negligible even for large memory reductions, since SMYRF can perform on par (or even better, e.g. GLUE) with dense attention, especially on downstream Natural Language Processing tasks, using a fraction of the original memory.

\section{Complexity analysis and speedups}
In the paper, we presented shortly the complexity of our algorithm. In this section, we explain it in more detail and we also include speed plots that demonstrate the effectiveness of SMYRF for long sequences.

\subsection{Complexity Analysis}
For the complexity analysis, we assume for simplicity that $|\mathcal Q|=|\mathcal K| = N$, i.e. the number of available queries is equal to the number of available keys.

We run the algorithm $H$ times (i.e. rounds of LSH). Each run has two stages:
\begin{itemize}
    \item[--] Clustering in L clusters (of equal size). For clustering, we hash all points with LSH which requires complexity $O(N)$ and then we sort points based on their hash, which requires complexity $O( N\cdot \log N)$. Overall, the complexity is $O(N \cdot \log N)$.
    \item[--] Within clusters attention. Attention within each cluster has quadratic cost with respect to the cluster size. Each cluster has size $\frac{N}{L}$, so the complexity of attention in a single cluster is $O(\frac{N^2}{L^2})$. We have $L$ such clusters, and thus the overall complexity is $O(\frac{N^2}{L})$.
\end{itemize}

The total complexity is: $O\left(H \cdot N \cdot \log N + H \cdot \frac{N^2}{L}\right)$. We choose $L = O(N)$, i.e. each query attends to a small constant number of keys. We obtain complexity: $O(H \cdot N \cdot \log N)$. 

\subsection{Speedups}
In this subsection, we present two speed plots to demonstrate the speed effectiveness of SMYRF for large sequences. The first plot, Figure \ref{fixedmemory}, shows elapsed time for SMYRF-BERT (base) GPU inference for various batch-sequence length configurations. In all these experiments $\textrm{batch size} \times N = 65$K, where $N$ denotes the sequence length. We underline that SMYRF has (almost) constant speed in all these configurations while the speed of dense attention decreases rapidly us the sequence length increases. Notably, SMYRF is already faster than dense attention in sequence length $1024$ tokens. The second plot, Figure \ref{fixedbatch}, shows seconds per iteration for SMYRF-BERT (base) GPU inference for various hashes-cluster configurations. In all these experiments, batch size is fixed to $1$. As shown, all different configurations significantly outperform (in terms of speed) dense attention as the sequence length increases.

\begin{figure}[!htp]
    \centering
    \includegraphics[width=\textwidth]{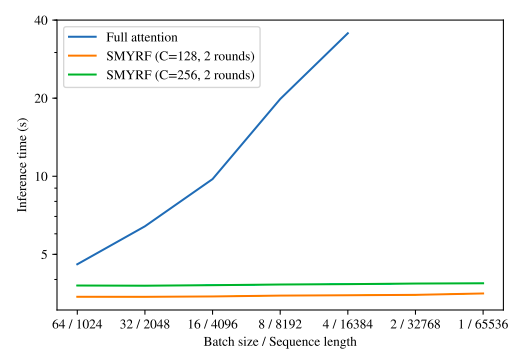}
    \caption{Elapsed time for SMYRF-BERT (base) GPU inference for various batch-sequence length configurations. Elapsed time for SMYRF is almost constant for all configurations. Elapsed time for dense attention worsens a lot as we increase the sequence length.}
    \label{fixedmemory}
\end{figure}

\begin{figure}[!htp]
    \centering
    \includegraphics[width=\textwidth]{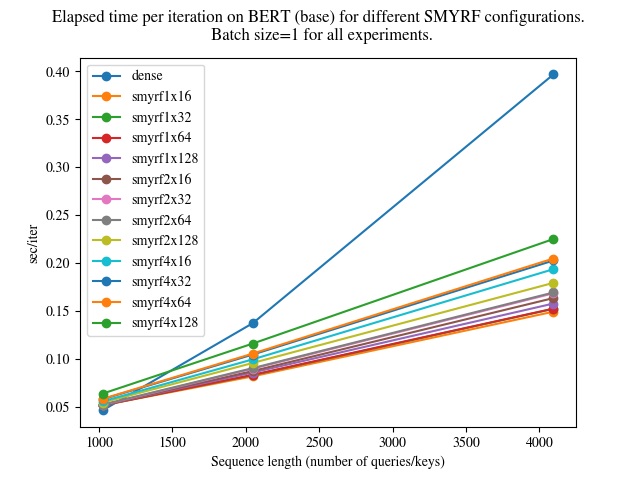}
    \caption{Seconds/iteration for SMYRF-BERT (base) GPU inference for various hashes-cluster configurations. In all experiments batch size is fixed to $1$. SMYRF has approximate the same speed with dense attention at 1024 tokens. However, as the number of tokens increases, SMYRF is significantly faster than dense attention.}
    \label{fixedbatch}
\end{figure}

\section{Experimental details}
\subsection{Natural Language Processing experiments}
In this section, we provide some details about the experimental settings for the Natural Language Processing experiments. 
\subsubsection{IMDB}
IMDB~\cite{IMDB} contains 25,000 train and 25,000 dev labeled movie reviews. The task is to identify if a given movie review is positive or negative. The average sentence length in IMDB is 300 tokens and the 95th percentile of context length is 705 tokens. In our experiments, we truncated/padded all sentences to $512$ tokens. For all our experiments, we trained for 3 epochs, with batch size 8. We used Adam~\cite{adam} as our optimizer with learning rate $3\cdot 10^{-5}$. The dataset is available publicly in this link: \href{https://ai.stanford.edu/~amaas/data/sentiment/}{https://ai.stanford.edu/~amaas/data/sentiment/}. The experiments on IMDB run on a single GPU provided by Google Colab.

\subsubsection{GLUE}
GLUE~\cite{GLUE} is a standard multitask benchmark for Natural Language Processing. For a full description of tasks, dataset statistics and files, please refer to the official website: \href{https://gluebenchmark.com/}{https://gluebenchmark.com/}. Following previous literature (e.g.~\cite{liu2019roberta, devlin2018bert, yang2019xlnet, clark2020electra}), for our GLUE experiments we truncate/pad all input sentences to $128$ tokens. For GLUE, we trained for 3 epochs at batch size 16, warming up for 10\% of the total training time.  The learning rate was selected via grid search among the values $\{5\cdot 10^{-5}, 3\cdot 10^{-5}, 2\cdot 10^{-5}\}$. We run the GLUE experiments on TPUs.

\subsection{Training SMYRF-BigGAN on Celeba-HQ}
In the paper we presented results for training SMYRF-BigGAN from scratch on Celeba-HQ~\cite{celeba}. As explained, we trained on Celeba-HQ (and not ImageNet~\cite{ImageNet}) in order to save computational resources. In this section, we provide the details for these experiments. First of all, as the name suggests, we used as the underlying model, BigGAN~\cite{biggan}. For our experiments we disabled BigGAN's hierarchical latent codes, shared embeddings and skip-z connections since Celeba-HQ has one single class (humans) and these architectural choices were introduced to model multiple classes (e.g. 1000 classes on ImageNet). We also found that for the single-class Celeba-HQ we didn't have to use very large batch sizes for stable training. For all our experiments, we used batch size $32$. Following the BigGAN paper, we used Two Time Scale Update Rule (TTUR)~\cite{FID} with Adam~\cite{adam} optimizer, G$_{\textrm{lr}}=2\cdot 10^{-4}$, D$_{\textrm{lr}}=5\cdot 10^{-5}$, $\beta_1=0$ and $\beta_2=0.999$.

BigGAN for resolutions $\{128\times 128, 256\times 256\}$, is trained with a single attention layer at resolution $64\times 64$. The authors mention that they stick attention to low resolution to save computational resources. We take advantage of SMYRF's reduced memory requirements to train with attention at resolution $128\times 128$ and $256\times 256$. For both experiments, we remove the dense attention layer and we add a SMYRF attention layer. Since our goal is to demonstrate the ability of SMYRF layers to train successfully from scratch, there is no reason to use higher (image) resolutions than the attention resolution and thus SMYRF is the final layer (before Tanh~\cite{tanh}) in the architecture. In other words, we train on image resolutions $128\times 128, 256\times 256$ respectively. Training on resolution $128\times 128$ has the side-benefit that we can compare directly with the original BigGAN model (with dense attention at $64\times 64$). As we demonstrated in the Experiments section of the paper, moving attention from $64\times 64$ to $128\times 128$ can lead to $\approx 4\%$ FID~\cite{FID} improvement after $120$K training steps\footnote{We note that in order to save computational resources we stopped training for both models on 120K iterations, before mode-collapse. That means that further training could possibly lead to even better FID scores for both models.}. We present random generated images from SMYRF-BigGAN with attention at resolution $256\times 256$ at Figure~\ref{gen256}. As explained in the Things that did not work section, training with SMYRF from scratch is harder as the sequence length increases. The main reason is that during the early stages of training attention maps are not sparse and thus our approximation's algorithm output is not close to the dense attention output. We noticed that the overall performance of SMYRF-BigGAN-256 is lower compared to SMYRF-BigGAN-128 and the generated images seem slightly less realistic. Despite the aforementioned shortcomings, this experiment demonstrated that it is possible to successfully train an attention GAN with attention at $256\times 256$ resolution on a \textit{single} TPUv3-8 device. The training at $128\times 128$ resolution lasts approximately 1.5 day and at $256\times 256$ resolution approximately 2 days.
\begin{figure}
    \centering
    \includegraphics[width=0.5\textwidth]{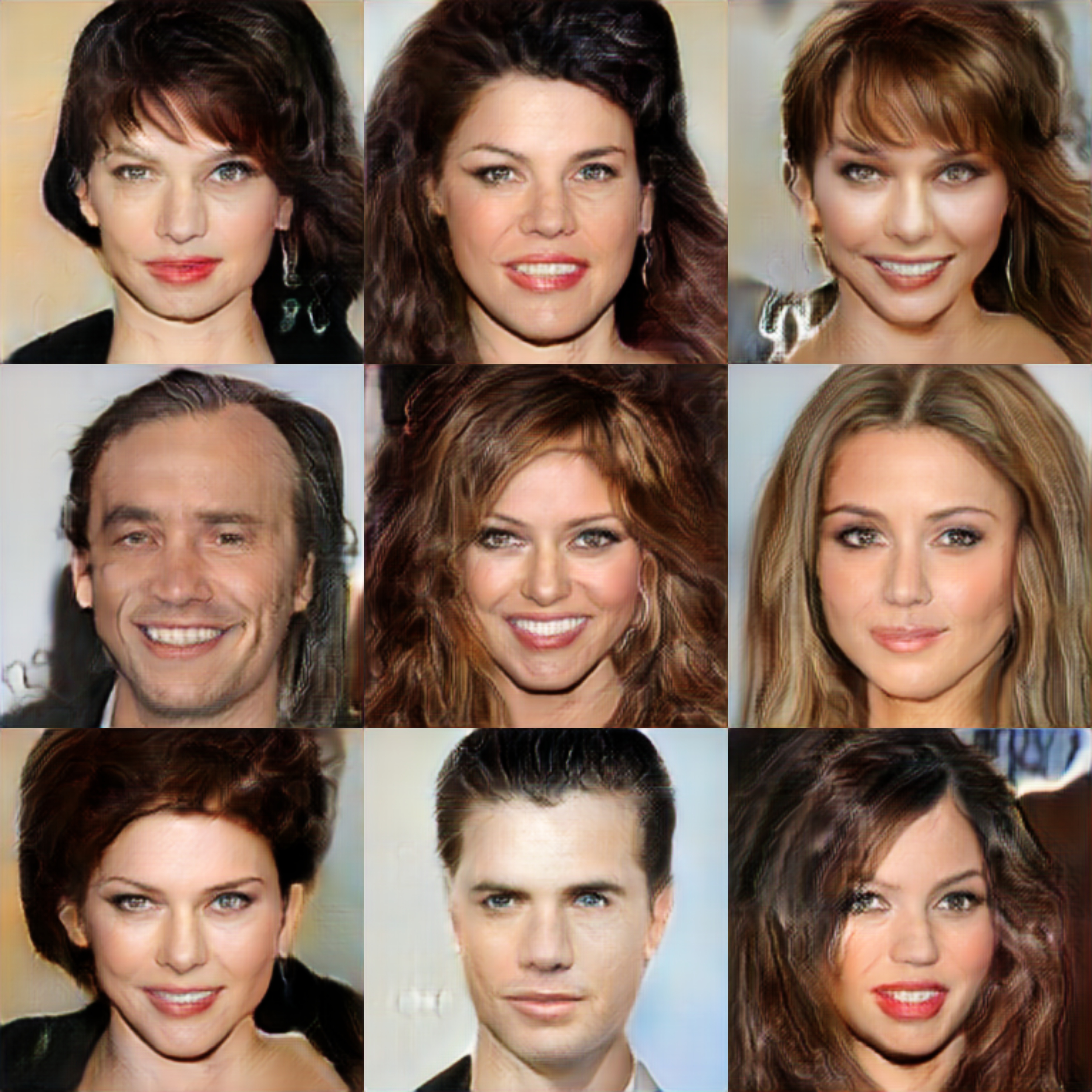}
    \caption{Generated images from SMYRF-BigGAN on Celeba-HQ-256.  Attention at $256\times256$. The trained model uses $50\%$ less memory compared to the memory dense attention would use.}
    \label{gen256}
\end{figure}

\section{Things that did not work}
In this section, we discuss some negative results we encountered in the process of writing this paper. Our goal is to share our experience with the research community about the observed shortcomings of some approaches so that future research can re-formulate them, reject them or even contradict our findings. We also include some suggestions on potential ways to alleviate such problems that we did not have the time to explore in this paper.

\subsection{Learning from scratch under extreme sparsity}
Our initial goal was to train a SMYRF model from scratch with extreme memory reductions, e.g. to the magnitude of $99\%$. Such reduction could enable the training of SMYRF-BigGAN with attention at $1024\times1024$. However, our preliminary experiments with BigGAN~\cite{biggan}, failed (mode-collapse very early in the training process). We tried to investigate this further and we found that during the early stages of the training the Frobenius norm of the difference between the SMYRF and the dense attention map is really high. We believe that this is due to the non-sparsity of the attention maps in the early stages of the training. It is also possible that their eigenvalues decay slower which means that their effective rank is higher compared to pre-trained models. One way to solve the problem is to dynamically adapt the memory reduction (e.g. by selecting the number of hashes) during the training. One way to achieve that is to use as many hashes as need to achieve a certain bound for the Frobenius norm. In the early stages of training, we expect that more hashes are needed for an accurate reconstruction. The number of hashes should decay as the training progresses and the attention maps become more sparse and have lower rank. One disadvantage of this approach is that at the early stages of the training, more memory is needed. However, we observed that the period of time in which the attention maps are not very sparse is minor compared to the whole training time for BigGAN and thus this approach can lead to significant savings. We aim to explore this more in the future.

\subsection{Better LSH based clustering schemes}
The biggest advantage of clustering with an LSH-based scheme is that the attention complexity is linear (compared to K-means clustering for example, see Routing Transformer~\cite{routing_transformer}). However, while inspecting SMYRF, we found that LSH-clustering is the biggest bottleneck to greater performances. For example, if each query attends to at its' top-k (in terms of inner product) keys (instead of the keys assigned with LSH), the performance improves considerably. Finding exactly the top-k keys for each query is expensive (especially in high dimensions) and thus this approach is not viable. However, this observation motivates research in finding even more effective LSH-based clustering schemes. Even though we tried other ALSH variants, we did not manage to find something that works better than our proposed transformations till now. We consider this problem an interesting future direction since ALSH has been widely explored only for the case of a single query and multiple keys. In this paper, we did the first step in extending this to multiple queries, but we are inclined to believe that further research can lead to even better results in this direction.

\bibliographystyle{unsrt}
\bibliography{SMYRF}
\end{document}